\documentclass[twoside,11pt]{article}
\usepackage{algorithm}
\usepackage{algorithmic}

\usepackage{jmlr2e}
\usepackage{lastpage}
\usepackage{multirow}        % safe after jmlr2e

\usepackage{amsmath}
\usepackage{amssymb}        % provides \square used by \qed
\usepackage{xcolor}          % NO [table] option
\usepackage{bm}
\usepackage{microtype}

\newcommand{\defeq}{\mathrel{\vcentcolon=}}
\providecommand{\vcentcolon}{%
  \mathrel{%
    \mathchoice{\vcenter{\hbox{$\displaystyle:$}}}%
               {\vcenter{\hbox{$:$}}}%
               {\vcenter{\hbox{$\scriptstyle:$}}}%
               {\vcenter{\hbox{$\scriptscriptstyle:$}}}}}

\providecommand{\qed}{\hfill\ensuremath{\square}}

\newcommand{\EIG}{\mathrm{EIG}}
\newcommand{\wEIG}{\widehat{\mathrm{EIG}}}
\newcommand{\calS}{\mathcal{S}}
\newcommand{\calA}{\mathcal{A}}
\newcommand{\calB}{\mathcal{B}}
\newcommand{\calR}{\mathcal{R}}
\newcommand{\calY}{\mathcal{Y}}
\newcommand{\calL}{\mathcal{L}}
\newcommand{\R}{\mathbb{R}}
\newcommand{\E}{\mathbb{E}}
\newcommand{\OPT}{\mathrm{OPT}}

\newcommand{\Vhard}{V^{\mathrm{hard}}_B}
\newcommand{\Vexp}{V^{\mathrm{exp}}_B}
\DeclareMathOperator{\Var}{Var}

\newcommand{\COHORTSIZE}{202{,}971}
\newcommand{\SPLIT}{70/15/15}

\jmlrheading{}{}{}{}{}{}{Bingham}

\ShortHeadings{Lagrangian Structure of Cost-Penalized AFA}{Bingham}
\firstpageno{1}

\begin{document}

\title{The Lagrangian Structure of Cost-Penalized\\
       Active Feature Acquisition}

\author{%
  \name Joseph Bingham
  \email jbingham@campus.technion.ac.il \\
  \addr Faculty of Computer Science \\
        Faculty of Biology \\
        Technion -- Israel Institute of Technology \\
        Haifa 3200003, Israel
}

\editor{To be assigned}

\maketitle

% ============================================================
\begin{abstract}%
Active Feature Acquisition (AFA) sequentially selects which
features to measure to maximize predictive accuracy subject
to a total acquisition budget.
A natural heuristic scores each candidate feature by
its Expected Information Gain minus a cost penalty,
$\EIG(j \mid x_{\calS}) - \lambda c_j$, and acquires
greedily (the \emph{EIG-Cost} policy), but without a
theoretical account of what it optimizes or how far it
falls short.
Our central result supplies one: an exact a-posteriori
\emph{certificate decomposition} that bounds the gap between
EIG-Cost and the optimal hard-budget policy by two
interpretable terms, a budget-slack term and a
greedy-suboptimality term $\Delta_\lambda$.
The certificate is \emph{fully computable whenever the
penalized optimum is tractable}, as on the small
heterogeneous-cost instances where we verify it exactly;
there the binding term $\Delta_\lambda$ is
near zero across $4{,}000$ exactly-solved cases.
On the large clinical problem $\Delta_\lambda$ is not
computable and the result is a diagnostic decomposition
without a numerical guarantee, a limitation we state
throughout.
The certificate rests on identifying EIG-Cost as one-step
greedy on the Lagrangian relaxation of the expected-cost AFA
problem, with $\lambda$ the dual multiplier of that
relaxation.
In the equal-cost case we further prove a \emph{structural
bound} on the greedy gap, $\Delta_\lambda \leq e^{-\gamma}
I(Y; X_{[d]})$ in terms of the submodularity ratio $\gamma$,
which vanishes in the near-submodular regime and turns the
certificate into a fully structural guarantee there.
Reaching this requires two supporting results we believe are
of independent interest: an adaptive chain-rule identity
establishing that cumulative step-wise EIG optimizes
\emph{transcript information} (the correct adaptive
objective, distinct from the mutual information of the
random terminal set); and a negative result that
$\calA \mapsto I(Y; X_{\calA})$ is not submodular even for
jointly Gaussian distributions, ruling out the
Nemhauser--Wolsey--Fisher bound and motivating the
submodularity-ratio analysis, with na\"ive Bayes identified
as the exact structure restoring submodularity.
We further bound the randomization gap between the
hard-budget and expected-cost optima, give a dual-sensitivity
observation (under local smoothness of the piecewise-linear
dual, the certificate inflates quadratically in multiplier
error near an interior minimiser), and give a
\emph{conditional} generator-shift bound on the CVAE
estimation error that holds given a bound on a forward
predictive KL, treated as a diagnostic quantity, not
as a transform of the training loss.
We illustrate these results on synthetic Gaussian
experiments (where every quantity is exactly computable),
AFA benchmarks (CKD, PhysioNet 2012), and a clinical
differential-diagnosis task on MIMIC-IV
(\COHORTSIZE{} admissions, $21$ conditions, Medicare CPT
costs), where EIG-Cost attains macro-F1
$0.180 \pm 0.012$ at budget \$30 while spending $6.6\%$
of full-panel cost.
The experiments test the assumptions and scope of each
result; we do not assert an end-to-end guarantee for
the general adaptive algorithm, which we do not claim.
\end{abstract}

\begin{keywords}
  active feature acquisition,
  expected information gain,
  budget-constrained optimization,
  Lagrangian relaxation,
  submodular functions,
  Monte Carlo estimation,
  clinical decision support,
  MIMIC-IV
\end{keywords}

% ============================================================
\section{Introduction}
\label{sec:intro}

Clinical decision-making under resource constraints is a
canonical instance of sequential information gathering.
A physician must decide which diagnostic tests to order,
each costing time and money, in order to reach an
accurate differential diagnosis as efficiently as possible.
This structure is shared by fraud detection, interactive
question answering, and scientific experiment design.

The general framework is \emph{Active Feature Acquisition}
(AFA) \citep{settles2009active}, a form of active learning
in which the feature space is the action space.
Prior work spans decision-theoretic value-of-information
methods \citep{krause2009}, deep reinforcement learning
agents \citep{janisch2019classification}, and
normalizing-flow imputors \citep{covert2023learning}.
The dominant information-theoretic approach selects features
greedily by their Expected Information Gain
\citep{ma2019eddi}.
The EIG-Cost method has shown promise in applied clinical
AFA; the present paper develops its theoretical foundation
and reports an independent experimental evaluation
(Section~\ref{sec:experiments}).

Specifically, EIG-Cost scores each unacquired feature $j$ as
\[
  \EIG(j \mid x_{\calS}) - \lambda\, c_j,
\]
where $c_j$ is the monetary cost and $\lambda \geq 0$ is a
cost-penalty coefficient.
Three critical questions were unanswered: why penalize EIG
linearly in cost? what guarantee does the greedy policy
carry? and how much does approximation error degrade it?

\paragraph{Contributions.}
We develop a theory of cost-aware adaptive feature
acquisition: we identify the correct adaptive objective
(transcript information), characterize cost-penalized greedy
acquisition as its Lagrangian one-step rule, pinpoint
exactly where classical submodular guarantees fail, and
delimit what can and cannot be certified for the resulting
policy.
\begin{enumerate}
  \item \textbf{An exact hard-budget certificate for EIG-Cost}
        (Section~\ref{sec:lagrangian}, Theorem~\ref{thm:lag_bound};
        our central result).
        We identify EIG-Cost as one-step greedy on the
        Lagrangian relaxation of the expected-cost AFA
        problem, with $\lambda$ the dual multiplier of that
        relaxation, and from this derive an exact
        \emph{two-term certificate}: the gap between EIG-Cost
        and the optimal hard-budget policy is bounded by a
        budget-slack term plus a greedy-suboptimality term
        $\Delta_\lambda$.
        Both terms are computable on tractable instances, and
        the binding term $\Delta_\lambda$ is empirically near
        zero (Section~\ref{sec:exp_hetcost}).
        We keep the hard-budget primal, the expected-cost
        relaxation, and the dual certificate as distinct
        objects throughout, and do \emph{not} claim greedy
        solves the relaxed optimization (NP-hard,
        \citealt{krause2005}) or that the general
        heterogeneous-cost adaptive algorithm carries a
        universal approximation ratio.
  \item \textbf{A computable bound on the
        hard-versus-expected-budget gap}
        (Proposition~\ref{thm:integrality_gap}).
        The one caveat the Lagrangian view cannot avoid
, that $\lambda$ is the dual of the expected-cost
        relaxation, not of the deployed hard-budget problem
, we close quantitatively: the randomization gap
        $\Vexp - \Vhard$ is at most the value rise of the
        single concave-envelope segment straddling the budget,
        equivalently $\lambda^*(c_2 - c_1)$, which vanishes
        when $B$ meets an achievable cost level.
  \item \textbf{A matched negative and positive submodularity
        result} (Section~\ref{sec:submodularity}).
        Negatively, $\calA \mapsto I(Y;\,X_{\calA})$ is not
        submodular (it exhibits strict complementarity)
        even for the additive Gaussian model
        $Y = \sum_j X_j + \eta$, ruling out direct
        application of the Nemhauser--Wolsey--Fisher $(1{-}1/e)$
        guarantee to AFA.
        Positively, $I(Y;\,X_{\calA})$ \emph{is} submodular
        under na\"ive Bayes structure
        (Proposition~\ref{prop:nb_submodular}), which
        identifies exactly the conditional-independence
        condition that makes the classical guarantee available.
        We parametrize the departure via the submodularity
        ratio \citep{das2011} and measure it on MIMIC-IV
        (Table~\ref{tab:gamma_mimic}).
\end{enumerate}
Two supporting results underpin the analysis: an adaptive
chain-rule identity (Lemma~\ref{prop:chain_rule}) equating
the expected cumulative step-wise EIG of any non-anticipating
adaptive policy with the information its acquisition
transcript carries about $Y$, the correct global object
for an adaptive policy, which differs from the mutual
information of the (random) terminal set, making
cumulative greedy EIG a coherent objective; and a bound on
the generator-shift component of the EIG estimation error
(Proposition~\ref{prop:cvae_error}) that vanishes as the
CVAE fits the true conditional, stated as a qualitative
convergence guarantee, without a tight quantitative
claim.

\paragraph{Scope of the guarantees.}
Two scope caveats apply throughout and we state them
up front.
First, the strong-duality result
(Proposition~\ref{prop:strong_duality}) is established for
\emph{finite-support} observations; continuous lab values
(as in MIMIC-IV) are handled by discretization, and we do
not pursue the infinite-dimensional constrained-MDP
machinery needed for the continuous case.
Second, the general heterogeneous-cost \emph{adaptive}
algorithm carries no universal approximation ratio; the
clean $(1-e^{-\gamma})$ guarantee (Corollary~\ref{thm:e2e})
holds only in the equal-cost, non-adaptive special case,
which we present as a sanity-check connection to classical
submodular greedy theory, and not as a guarantee for the
deployed method.
The practical content for the general algorithm is the
exact shortfall decomposition
(Theorem~\ref{thm:lag_bound}) together with the
\emph{a-posteriori} certificate obtained by measuring the
greedy gap $\Delta_\lambda$ where it is computable.

\paragraph{Paper organization.}
Section~\ref{sec:related} reviews related work.
Section~\ref{sec:formulation} formalizes the problem.
Section~\ref{sec:algorithm} describes EIG-Cost.
Section~\ref{sec:theory} contains the theoretical results.
Section~\ref{sec:experiments} presents empirical validation.
Section~\ref{sec:discussion} discusses limitations and
future work.
Section~\ref{sec:conclusion} concludes.
Full proofs appear in Appendix~\ref{app:proofs}.

% ============================================================
\section{Related Work}
\label{sec:related}

\paragraph{Active Feature Acquisition.}
The AFA problem was formalized by \citet{settles2009active};
see \citet{aronsson2025survey} for a recent survey.
Early work used decision-theoretic value-of-information
criteria \citep{krause2009} and Bayesian optimal stopping
\citep{lindley1956measure}.
Deep learning approaches include JAFA
\citep{janisch2019classification}, SHIM \citep{shim2018joint},
ACFlow / CAE \citep{covert2023learning}, DIME
\citep{gadgil2023dime}, and ODIN \citep{zannone2019odin}.
The closest prior method is EDDI \citep{ma2019eddi}, which
introduces the Partial VAE and greedy EIG selection;
EIG-Cost extends EDDI by incorporating monetary costs and
a hard budget constraint.
Optimal subset selection for information gain is NP-hard
\citep{krause2005}, establishing that polynomial-time
approximations are necessary.

\paragraph{Submodular optimization and information gain.}
\citet{nemhauser1978} proved the $(1-1/e)$ approximation
ratio for greedy maximization of monotone submodular
functions under cardinality constraints.
\citet{krause2005} applied this to sensor placement using
submodularity of the Gaussian process entropy $H(X_{\calA})$
and mutual information $I(X_{\calA};\,X_{V\setminus\calA})$.
A central contribution here is that these results do
\emph{not} carry over to $I(Y;\,X_{\calA})$ with a fixed
target $Y$.
For knapsack-constrained submodular maximization the
density greedy achieves $(1/2)(1-1/e)$ \citep{khuller1999}.
\citet{das2011} introduced the submodularity ratio for
parametric guarantees in near-submodular settings; we adapt
this to AFA.
Adaptive submodularity \citep{golovin2011} extends
approximation results to sequential decisions, though the
required conditions are not satisfied by EIG in general
(see the remark following Lemma~\ref{lem:synergy}).
The broader literature on approximate and robust submodular
observation selection \citep{krause2008robust} develops
guarantees that degrade gracefully with the distance from
submodularity; the substitutes/complements dichotomy of our
Lemma~\ref{lem:synergy} is the information-theoretic analog
of complementarity in the theory of supermodular games and
monotone comparative statics \citep{topkis1998}.

\paragraph{Budget-constrained Bayesian experimental design.}
Lagrangian relaxation of information-gathering budgets
appears in \citet{foster2021deep} and
\citet{rainforth2024modern}.
The economics of information establishes related duality
results \citep{radner1984}.
The Lagrangian dual interpretation of the EIG cost-penalty
formulation is, to our knowledge, new.

\paragraph{Clinical AFA.}
Clinical cost-aware test ordering has been studied in
\citet{kachuee2019opportunistic}.
A recurring phenomenon in cost-aware AFA is \emph{baseline
collapse}, in which many methods degenerate to acquiring
nothing under a cost penalty; we observe it in our
evaluation (Section~\ref{sec:exp_mimic}) and explain it
theoretically via Corollary~\ref{cor:collapse}, which
formalizes it as a direct consequence of the voluntary-stop
condition.

% ============================================================
\section{Problem Formulation}
\label{sec:formulation}

\begin{definition}[Cost-Aware AFA Instance]
\label{def:afa}
A \emph{cost-aware AFA instance} is a tuple
$(\mathcal{X}, \mathcal{Y}, p, \{G_1,\dots,G_m\}, c, B)$
where $X = (X_1,\ldots,X_d) \in \mathcal{X}$,
$Y \in \mathcal{Y}$ with joint distribution
$p(Y,X_1,\ldots,X_d)$.
The \emph{acquisition items} are disjoint \emph{panels}
$G_1,\dots,G_m$ partitioning a subset of $[d]$, each with
cost $c_i \geq 0$; acquiring panel $G_i$ reveals all its
member features jointly.
Zero-cost pre-acquired variables (e.g.\ admission vitals)
form a distinguished panel $G_0$ with $c_0 = 0$ that is
included in the initial state and not treated as an
acquisition action, so every \emph{acquirable} panel has
$c_i > 0$ and ratio scores $\EIG/c_i$ are well-defined.
When each panel is a singleton and all costs are positive
this reduces to the per-feature setting; we use ``feature''
and ``panel'' interchangeably where the distinction does not
matter.
A \emph{partial observation} is $(x_{\calS},\calS)$ with
$\calS \subseteq [d]$ and
$x_{\calS} \in \prod_{j \in \calS} \mathcal{X}_j$, and
$B > 0$.
\end{definition}

\begin{definition}[Adaptive Policy]
\label{def:policy}
An \emph{adaptive policy} $\pi$ maps partial observations
to the next feature index or a \textsc{stop} action, with
each choice measurable with respect to the observed history.
It induces a random terminal set $\calS_T^\pi$.
Let $\Pi_B$ denote the policies whose acquisition cost
satisfies the budget almost surely,
$\sum_{j\in\calS_T^\pi} c_j \leq B$.
\end{definition}

\noindent
The value of an adaptive policy is the information its
acquisition transcript $H_T^\pi$ carries about $Y$,
\[
  V(\pi) \;=\; I(Y;\, H_T^\pi)
         \;=\; H(Y) - \E\!\left[H(Y \mid H_T^\pi)\right],
\]
the expected reduction in posterior entropy
(Lemma~\ref{prop:chain_rule}); we write
$C(\pi) = \E[\sum_{j\in\calS_T^\pi} c_j]$ for the expected
cost.
For a \emph{non-adaptive} policy (fixed acquisition order
and fixed cardinality) the transcript and the acquired-set
value coincide, $V(\pi) = I(Y; X_{\calS_T})$; for adaptive
policies they differ (Remark~\ref{rem:transcript_vs_set}),
and the transcript form is the correct objective.
The deployed \emph{hard-budget} AFA problem constrains cost
on every episode:
\begin{equation}
  \Vhard \;=\; \max_{\pi \in \Pi_B}\; V(\pi)
  \quad\text{subject to}\quad
  \sum_{j \in \calS_T^\pi} c_j \leq B \text{ a.s.}
\label{eq:afa}
\end{equation}
The maximization is over deterministic adaptive policies
whose acquisition cost never exceeds $B$ on any episode.
We will also use its \emph{expected-cost randomized
relaxation}, $\Vexp = \max_\pi\{V(\pi) : C(\pi) \leq B\}$
over randomized policies with the budget imposed only in
expectation ($C(\pi) = \E[\sum_{j\in\calS_T^\pi} c_j]$).
These are distinct: $\Vhard \leq \Vexp$ always, and the two
differ by the randomization gap bounded in
Proposition~\ref{thm:integrality_gap}.
Strong duality (Section~\ref{sec:lagrangian}) holds for
$\Vexp$, not for $\Vhard$; we keep the two notations
separate throughout to avoid conflating them.

\begin{definition}[Expected Information Gain]
\label{def:eig}
\begin{equation}
  \EIG(j \mid x_{\calS})
  \;\defeq\;
  I(Y;\,X_j \mid X_{\calS} = x_{\calS})
  \;=\;
  H(Y \mid x_{\calS})
    - \E_{X_j \mid x_{\calS}}
      \!\bigl[H(Y \mid x_{\calS}, X_j)\bigr].
\label{eq:eig}
\end{equation}
Throughout, $H(Y \mid x_{\calS})$ abbreviates the entropy
of $Y$ under the conditional law $Y \mid X_{\calS} = x_{\calS}$
at the observed realization $x_{\calS}$, whereas
$H(Y \mid X_{\calS})$ (with capital $X$) denotes the
expectation of that entropy over $X_{\calS}$; we use the
lower-case form whenever a specific realization is fixed.
The estimated counterpart of any information quantity is
written with a hat, e.g.\ $\wEIG$ in \eqref{eq:mc_eig}.
\end{definition}

\begin{definition}[EIG-Cost Policy]
\label{def:eig_cost}
The \emph{EIG-Cost policy} with penalty $\lambda \geq 0$
selects at step $t$:
\begin{equation}
  j_t^* = \arg\max_{\substack{j \notin \calS_{t-1} \\
           c_j \leq B - C_{t-1}}}
  \bigl[\EIG(j \mid x_{\calS_{t-1}}) - \lambda c_j\bigr],
\label{eq:eig_cost_rule}
\end{equation}
terminating when no eligible feature has a positive
penalized score or when budget is exhausted.
\end{definition}

The central problem is the budget-constrained transcript
objective \eqref{eq:afa}:
\begin{equation}
  \Vhard = \max_{\pi \in \Pi_B}\;
  I(Y;\,H_T^\pi),
\label{eq:primal}
\end{equation}
with $V(\pi) = I(Y; H_T^\pi)$ the expected posterior-entropy
reduction (Lemma~\ref{prop:chain_rule}).

% ============================================================
\section{The EIG-Cost Algorithm}
\label{sec:algorithm}

Figure~\ref{fig:architecture} shows the EIG-Cost pipeline.
Algorithm~\ref{alg:eig_cost} states the procedure formally.
EIG \eqref{eq:eig} is approximated via $K$ Monte Carlo
samples from a CVAE $p_\theta(X_j \mid x_{\calS})$:
\begin{equation}
  \wEIG(j \mid x_{\calS})
  = H(Y \mid x_{\calS})
    - \frac{1}{K} \sum_{k=1}^K
      H\!\bigl(Y \mid x_{\calS},\,\hat{x}_j^{(k)}\bigr),
  \quad
  \hat{x}_j^{(k)} \sim p_\theta(X_j \mid x_{\calS}).
\label{eq:mc_eig}
\end{equation}
The posterior and conditional entropies are evaluated by a
mask-aware classifier $p_\phi(Y \mid x_{\calS}, m_{\calS})$
with observation mask $m_{\calS} \in \{0,1\}^d$.
In our clinical instantiation, the CVAE
is a TabularMAE architecture, the classifier is a
mask-aware MLP, and costs are 2026 Medicare CPT rates.

\begin{figure}[htb]
  \centering
  \includegraphics[width=\linewidth]{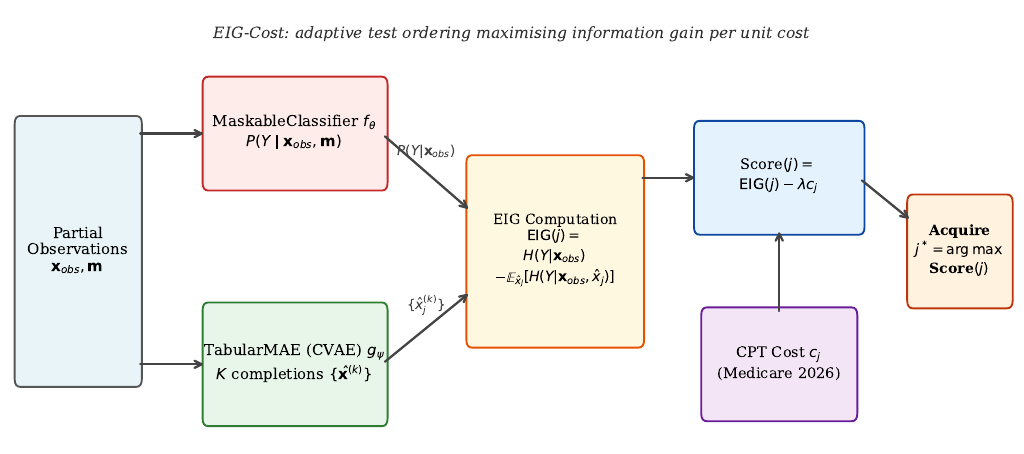}
  \caption{\textbf{EIG-Cost pipeline.}
    Partial observations $(x_{\calS}, m_{\calS})$ are fed
    simultaneously to the MaskableClassifier $f_\theta$
    (computing the posterior $P(Y \mid x_{\calS})$) and to
    the TabularMAE CVAE $g_\psi$ (generating $K$ Monte Carlo
    completions $\{\hat{x}^{(k)}\}$).
    The EIG Computation block evaluates Eq.~\eqref{eq:eig};
    the CPT cost $c_j$ is subtracted (scaled by $\lambda$)
    to yield the acquisition score.
    The highest-scoring panel is acquired and the loop
    repeats until budget exhaustion.}
  \label{fig:architecture}
\end{figure}

\begin{algorithm}[t]
\caption{EIG-Cost}
\label{alg:eig_cost}
\begin{algorithmic}[1]
\REQUIRE CVAE $p_\theta$, classifier $p_\phi$,
         costs $c$, budget $B$,
         penalty $\lambda \geq 0$, samples $K$
\STATE $\calS \gets \emptyset$;\; $C \gets 0$
\WHILE{$[d] \setminus \calS \neq \emptyset$}
  \FOR{each $j \in [d] \setminus \calS$ with
       $c_j \leq B - C$}
    \STATE Compute $\wEIG(j \mid x_{\calS})$
           via \eqref{eq:mc_eig};\;
           $s_j \gets \wEIG(j \mid x_{\calS}) - \lambda c_j$
  \ENDFOR
  \IF{$\max_j\, s_j \leq 0$}
    \STATE \textbf{break}
  \ENDIF
  \STATE $j^* \gets \arg\max_j\, s_j$;\;
         observe $x_{j^*}$
  \STATE $\calS \gets \calS \cup \{j^*\}$;\;
         $C \gets C + c_{j^*}$
\ENDWHILE
\RETURN $\hat{y} = \arg\max_y\, p_\phi(y \mid x_{\calS})$
\end{algorithmic}
\end{algorithm}

% ============================================================
\section{Theoretical Analysis}
\label{sec:theory}

Our analysis has one destination: an exact a-posteriori
\emph{certificate decomposition} on the hard-budget
optimality gap of EIG-Cost (Theorem~\ref{thm:lag_bound}),
which is the paper's central theoretical result.
The route there has three steps.
First, we establish the correct adaptive objective
transcript information, and show cumulative greedy EIG
optimizes it (Section~\ref{sec:chain_rule}); this is the
identity that makes the greedy rule coherent.
Second, we show why the classical route to a guarantee is
unavailable: $I(Y; X_{\calA})$ is not submodular, so the
Nemhauser--Wolsey--Fisher bound does not apply
(Section~\ref{sec:submodularity}).
Third, and centrally, we identify EIG-Cost as
greedy on the Lagrangian relaxation and derive the
certificate and its sensitivity
(Section~\ref{sec:lagrangian}).
The remaining subsections are supporting: an equal-cost
special case where a clean constant-factor ratio does hold,
an estimation-error bound for the CVAE, and a
characterization of voluntary stopping.

\subsection{The Adaptive Chain-Rule Identity}
\label{sec:chain_rule}

For an adaptive policy the terminal acquired set is a
\emph{random}, value-dependent object, so care is needed to
identify what cumulative stepwise EIG actually equals.
Let the \emph{transcript} after $t$ steps be the ordered
record of chosen indices and observed values,
\[
  H_t \;=\; \bigl(j_1, X_{j_1}, \ldots, j_t, X_{j_t}\bigr),
\]
and let $H_T$ be the transcript at the (random) stopping
time $T$.
Under a non-anticipating policy the indices $j_1,\ldots,j_t$
and the stopping decision are deterministic functions of the
observed values, so $H_t$ and the observed value vector
$X_{\calS_t}$ generate the same information about $Y$.

\begin{lemma}[Adaptive Chain-Rule Identity]
\label{prop:chain_rule}
For any non-anticipating adaptive policy $\pi$ acquiring
distinct features,
\begin{equation}
  \E\!\left[
    \sum_{t=1}^{T}
    \EIG\!\left(j_t \mid X_{\calS_{t-1}}\right)
  \right]
  \;=\;
  I\!\left(Y;\, H_T\right)
  \;=\;
  H(Y) - \E\!\left[H\!\left(Y \mid H_T\right)\right].
\label{eq:chain_rule_identity}
\end{equation}
The right-hand side is the information about $Y$ contained in
the full acquisition transcript, i.e.\ the expected reduction
in posterior entropy achieved by the policy.
\end{lemma}

\begin{proof}
Write $\EIG(j_t \mid X_{\calS_{t-1}}) =
I(Y; X_{j_t} \mid H_{t-1})$: at each step the local expected
information gain, evaluated at the realized history, is the
conditional mutual information between $Y$ and the newly
observed value given the transcript so far.
Summing and taking expectations, the chain rule for mutual
information applied to the (adapted) sequence
$X_{j_1},\ldots,X_{j_T}$ conditioned on the growing
transcript gives, for the stopped sequence,
\[
  \E\!\left[\sum_{t=1}^{T} I(Y; X_{j_t} \mid H_{t-1})\right]
  = I(Y; H_T),
\]
because the increments telescope: $I(Y; H_t) - I(Y; H_{t-1})
= I(Y; (j_t, X_{j_t}) \mid H_{t-1}) = I(Y; X_{j_t} \mid
H_{t-1})$, where the last equality uses that $j_t$ is
$H_{t-1}$-measurable (non-anticipating policy), so revealing
the index adds no information about $Y$.
Summing over $t$ from $1$ to the stopping time $T$ and taking
expectations yields \eqref{eq:chain_rule_identity}; the
optional stopping is handled automatically because each
realized path contributes its own telescoped sum
$I(Y; H_T) = \sum_{t \leq T} [I(Y;H_t) - I(Y;H_{t-1})]$
before the outer expectation.
\end{proof}

\begin{remark}[Transcript information, not random-set MI]
\label{rem:transcript_vs_set}
The right-hand side of \eqref{eq:chain_rule_identity} is
$I(Y; H_T)$, \emph{not} $\E[I(Y; X_{\calS_T})]$, the expected
mutual information of the random terminal set.
For an adaptive policy these differ: the terminal set
$\calS_T$ is itself random and correlated with the data, so
averaging the set-function $I(Y; X_{\calS})$ over the induced
distribution of $\calS_T$ discards the information carried by
\emph{which} stopping decision was made.
A minimal example: let $X_1, X_2 \sim \mathrm{Bernoulli}(1/2)$
i.i.d.\ and $Y = X_2$ if $X_1 = 0$, else $Y = 0$; the policy
acquires $X_1$, then $X_2$ only if $X_1 = 0$.
Here cumulative stepwise EIG $= I(Y; H_T) = H(Y) \approx
0.811$ bits, whereas
$\E[I(Y; X_{\calS_T})] = \tfrac12 I(Y; X_1, X_2) + \tfrac12
I(Y; X_1) \approx 0.561$ bits, a gap of exactly the
$0.25$ bits encoded by the value-dependent stopping.
The transcript identity is the correct global object and is
what the greedy analysis of Section~\ref{sec:lagrangian}
uses.
\end{remark}

\begin{corollary}
\label{cor:lagrangian_reward}
For any non-anticipating adaptive policy $\pi$,
\[
  \E\!\left[
    \sum_{t=1}^{T}
    \bigl(\EIG(j_t \mid X_{\calS_{t-1}}) - \lambda c_{j_t}\bigr)
  \right]
  = I(Y;\,H_T) - \lambda\,\E\!\left[\textstyle\sum_{j\in\calS_T^\pi}c_j\right]
  = I(Y; H_T) - \lambda\,C(\pi).
\]
Thus the cumulative penalized stepwise EIG equals the
transcript information minus the expected total cost:
the penalized objective is additive over steps and its
expectation is exactly the transcript-information objective
penalized by cost.
\end{corollary}

\begin{proof}
Apply Lemma~\ref{prop:chain_rule} to the information term;
the cost term is a deterministic sum along each realized
path, whose expectation is $C(\pi) = \E[\sum_{j\in\calS_T}
c_j]$ by linearity.
\end{proof}

\subsection{Submodularity: A Necessary Negative Result}
\label{sec:submodularity}

\begin{definition}[Submodularity]
$f : 2^{[d]} \to \R$ is \emph{submodular} if
$f(\calA \cup \{j\}) - f(\calA)
 \geq f(\calB \cup \{j\}) - f(\calB)$
for all $\calA \subseteq \calB$, $j \notin \calB$.
For $f(\calA) = I(Y;\,X_{\calA})$ this is equivalent to
$I(Y;\,X_j \mid X_{\calA}) \geq I(Y;\,X_j \mid X_{\calB})$.
\end{definition}

\begin{lemma}[Monotonicity of $I(Y; X_{\calA})$]
\label{lem:monotone}
For any joint distribution and any
$\calA \subseteq \calB \subseteq [d]$,
$I(Y; X_{\calA}) \leq I(Y; X_{\calB})$.
\end{lemma}

\begin{proof}
By the chain rule,
$I(Y; X_{\calB}) = I(Y; X_{\calA})
 + I(Y; X_{\calB\setminus\calA} \mid X_{\calA})$,
and conditional mutual information is non-negative, so the
second term is $\geq 0$.
\end{proof}

\noindent
Monotonicity holds in full generality (no submodularity
required) and is the hypothesis needed to invoke the
submodularity-ratio guarantee of
Proposition~\ref{prop:gamma_approx}.

\begin{proposition}[Submodularity Fails in General]
\label{prop:not_submodular}
There exist discrete and jointly Gaussian distributions for
which $f(\calA) = I(Y;\,X_{\calA})$ is not submodular.
\end{proposition}

\begin{proof}
\emph{XOR.}
$Y = X_1 \oplus X_2$, $X_1,X_2 \sim \mathrm{Ber}(1/2)$:
$I(Y;\,X_1) = I(Y;\,X_2) = 0$ but $I(Y;\,X_1,X_2) = 1$.

\emph{Gaussian.}
Let $Y = X_1 + X_2 + \eta$ with
$X_1, X_2, \eta \sim N(0,1)$ independent.
For jointly Gaussian $(Y, X_{\calA})$,
$I(Y; X_{\calA}) = \tfrac{1}{2}\log
   \bigl(\Var(Y)/\Var(Y \mid X_{\calA})\bigr)$.
Here $\Var(Y) = 3$.
Conditioning on a single feature, say $X_1$, leaves
$Y \mid X_1 = X_2 + \eta + \text{const}$ with
$\Var(Y \mid X_1) = 2$, so
\[
  f(\{1\}) = f(\{2\})
  = \tfrac{1}{2}\log\tfrac{3}{2} \approx 0.203 .
\]
Conditioning on both features leaves
$Y \mid X_1, X_2 = \eta + \text{const}$ with
$\Var(Y \mid X_1, X_2) = 1$, so
\[
  f(\{1,2\}) = \tfrac{1}{2}\log 3 \approx 0.549 .
\]
Then $f(\{1\}) + f(\{2\}) \approx 0.406 < 0.549 = f(\{1,2\})$,
which violates the submodular inequality
(with $\calA = \emptyset$, $\calB = \{2\}$, $j = 1$):
the marginal value of $X_1$ \emph{increases} once $X_2$ is
observed, the signature of complementarity.
\end{proof}

\begin{remark}
This differs from prior submodularity results in an
important way.
\citet{krause2005} prove submodularity of
$I(X_{\calA};\,X_{V\setminus\calA})$ for Gaussian process
sensors, where the target is the set of unobserved sensor
locations.
AFA instead concerns $I(Y;\,X_{\calA})$ with a
\emph{fixed} target $Y$, and this quantity is not
submodular even for Gaussians
(Proposition~\ref{prop:not_submodular}).
The Gaussian-process submodularity result does not transfer
to AFA, and any approximation guarantee that invokes it is
misapplied.
\end{remark}

\begin{lemma}[Synergy Decomposition]
\label{lem:synergy}
For $\calA \subseteq \calB$, $j \notin \calB$:
\[
  I(Y;\,X_j \mid X_{\calA}) - I(Y;\,X_j \mid X_{\calB})
  = I(Y;\,X_{\calB\setminus\calA} \mid X_{\calA})
    - I(Y;\,X_{\calB\setminus\calA} \mid X_{\calA}, X_j).
\]
Call $X_j$ and $X_{\calB\setminus\calA}$ \emph{substitutes}
when the left-hand side is positive (equivalently,
$X_j$'s marginal value weakly decreases as more features
are observed, the diminishing-returns direction) and
\emph{complements} when it is negative
($X_j$'s marginal value increases with more observed
features, violating diminishing returns).
Submodularity of $I(Y;\,X_{\calA})$ is exactly the
statement that all feature pairs are substitutes in this
sense.
\end{lemma}

\begin{proof}
Apply the mutual-information chain rule to the joint term
$I(Y;\,X_j, X_{\calB\setminus\calA} \mid X_{\calA})$ in two
orders.
Ordering $X_j$ first:
\[
  I(Y;\,X_j, X_{\calB\setminus\calA} \mid X_{\calA})
  = I(Y;\,X_j \mid X_{\calA})
  + I(Y;\,X_{\calB\setminus\calA} \mid X_{\calA}, X_j).
\]
Ordering $X_{\calB\setminus\calA}$ first:
\[
  I(Y;\,X_j, X_{\calB\setminus\calA} \mid X_{\calA})
  = I(Y;\,X_{\calB\setminus\calA} \mid X_{\calA})
  + I(Y;\,X_j \mid X_{\calA}, X_{\calB\setminus\calA}).
\]
Since $X_{\calA} \cup (X_{\calB\setminus\calA}) = X_{\calB}$,
the last conditional equals
$I(Y;\,X_j \mid X_{\calB})$.
Equating the two expansions and rearranging,
\[
  I(Y;\,X_j \mid X_{\calA}) - I(Y;\,X_j \mid X_{\calB})
  = I(Y;\,X_{\calB\setminus\calA} \mid X_{\calA})
    - I(Y;\,X_{\calB\setminus\calA} \mid X_{\calA}, X_j),
\]
which is the claim.
The characterization of submodularity follows by taking
$\calB = \calA \cup \{i\}$ for each $i$ and comparing with
Definition~\ref{def:submod_ratio}.
\end{proof}

\begin{remark}[Adaptive submodularity does not apply]
One might hope to recover guarantees via adaptive
submodularity \citep{golovin2011}, which yields
$(1-1/e)$ bounds for adaptive greedy on adaptive monotone,
adaptive submodular objectives.
The obstruction is the same as in
Proposition~\ref{prop:not_submodular}: adaptive
submodularity requires the expected marginal benefit of an
item to be non-increasing in the observed history, which
is precisely the substitutes condition of
Lemma~\ref{lem:synergy}.
The Gaussian counterexample $Y = X_1 + X_2 + \eta$
violates it, because observing $X_1$ \emph{raises} the
conditional information $X_2$ carries about $Y$.
Adaptive submodularity therefore fails for AFA under the
same complementarity that defeats ordinary submodularity,
and the submodularity-ratio route we take is the
appropriate relaxation.
\end{remark}

\begin{proposition}[Submodularity under Conditional Independence]
\label{prop:nb_submodular}
If features are conditionally independent given $Y$ (na\"ive
Bayes), then $I(Y;\,X_{\calA})$ is monotone submodular.
\end{proposition}

\begin{proof}
See Appendix~\ref{app:nb_proof}.
\end{proof}

\begin{definition}[Submodularity Ratio, \citealt{das2011}]
\label{def:submod_ratio}
\begin{equation}
  \gamma_{f,k}
  = \min_{\substack{\calA \subseteq [d],\;
                    \calR \subseteq [d]\setminus\calA,\;
                    |\calR|\leq k,\;
                    f(\calA\cup\calR) > f(\calA)}}
    \frac{\sum_{j \in \calR}
            [f(\calA \cup \{j\}) - f(\calA)]}
         {f(\calA \cup \calR) - f(\calA)},
\label{eq:submod_ratio_def}
\end{equation}
where the minimum is taken over pairs with strictly positive
denominator $f(\calA\cup\calR) - f(\calA) > 0$ (pairs with
zero denominator are excluded, as the ratio is undefined and
the corresponding greedy step is never taken); if no such
pair exists we set $\gamma_{f,k} = 1$.
For monotone $f$ the raw ratio is $\geq 0$; a submodular $f$
satisfies $\gamma_{f,k} \geq 1$, and we clip to
$\min(\gamma_{f,k}, 1)$ so that $\gamma_{f,k} \in [0,1]$ with
$\gamma_{f,k} = 1$ iff $f$ is submodular over the relevant
sets and $\gamma_{f,k} < 1$ indicating feature synergy.
Reported empirical estimates near or slightly above $1$
(e.g.\ $1.001$) are Monte Carlo estimates of the pre-clip
ratio and correspond to the submodular/near-submodular
regime $\gamma_{f,k} = 1$.
\end{definition}

\begin{proposition}[\citealt{das2011}]
\label{prop:gamma_approx}
For monotone $f$ with submodularity ratio $\gamma_{f,k}
\geq \gamma$, greedy maximization of
$\max_{|\calA|\leq k}f(\calA)$ achieves
$f(\calS_k^{\mathrm{greedy}}) \geq (1-e^{-\gamma})\OPT_k$.
\end{proposition}

\subsection{EIG-Cost as Lagrangian Greedy: the Hard-Budget
            Certificate}
\label{sec:lagrangian}

\paragraph{Hard constraint versus soft penalty.}
Two budget formulations appear in this paper and it is
important to keep them distinct.
The \emph{primal} problem \eqref{eq:afa} and the EIG-Cost
algorithm (Definition~\ref{def:eig_cost}, Algorithm~1)
impose a \emph{hard} budget: acquisition halts once the
cumulative spend would exceed $B$.
The \emph{Lagrangian relaxation} \eqref{eq:lagrangian}
below replaces this hard constraint with a \emph{soft}
per-unit-cost penalty $\lambda$, turning the constrained
problem into an unconstrained penalized one.
The two are linked by duality: the penalty $\lambda$ is the
dual multiplier of the \emph{expected-cost relaxation} of
the budget constraint, and at the optimal
$\lambda^* = V'(B)$ the soft-penalized optimum matches the
expected-cost optimum $\Vexp$ (not the hard-budget optimum
$\Vhard$; Proposition~\ref{prop:strong_duality}), the two
differing by the randomization gap of
Proposition~\ref{thm:integrality_gap}.
The EIG-Cost \emph{scoring rule}
$\EIG(j) - \lambda c_j$ is the greedy step of the soft
problem; the hard budget in the algorithm acts as an
additional feasibility cap that binds only when the
per-step penalty has not already halted acquisition.
In the voluntary-stopping regime that EIG-Cost operates in
(Section~\ref{sec:exp_mimic}), the soft penalty halts
acquisition before the hard cap binds, so the two
formulations coincide in practice.

\begin{definition}[Lagrangian Relaxation]
\label{def:lagrangian}
For $\lambda \geq 0$:
\begin{equation}
  L(\lambda)
  = \max_{\pi}\;
    \left[
      I(Y;\, H_T^\pi)
      - \lambda\,\E\!\Bigl[\textstyle\sum_{j \in \calS_T^\pi} c_j\Bigr]
    \right].
\label{eq:lagrangian}
\end{equation}
\end{definition}

\begin{proposition}[EIG-Cost is the feasibility-capped
Lagrangian Greedy Step]
\label{prop:lagrangian_greedy}
Under the transcript objective $V(\pi) = I(Y; H_T^\pi)$
(Lemma~\ref{prop:chain_rule}), the EIG-Cost rule
\eqref{eq:eig_cost_rule} selects, at each realized history,
the feature maximizing the one-step expected increment of
the penalized objective
$J_\lambda(\pi) = V(\pi) - \lambda C(\pi)$ \emph{among the
actions feasible under the remaining hard budget}
($c_j \leq B - C_{t-1}$).
Writing $\pi^{\mathrm{soft}}_\lambda$ for the unconstrained
one-step greedy policy of $J_\lambda$ (no feasibility cap)
and $\pi^\lambda \defeq \pi^{\mathrm{hard}}_{\lambda,B}$ for
the deployed feasibility-capped policy, the two coincide
until the first step at which the unconstrained greedy
choice would exceed the remaining budget; thereafter
$\pi^\lambda$ restricts to affordable panels.
In either case, by weak duality
$\Vhard \leq L(\lambda) + \lambda B$ for all $\lambda \geq 0$.
\end{proposition}

\begin{proof}
By the additive decomposition of the penalized objective
(Corollary~\ref{cor:lagrangian_reward}), $J_\lambda(\pi)$ is
the expectation of the cumulative per-step reward
$\EIG(j_t \mid X_{\calS_{t-1}}) - \lambda c_{j_t}$.
At a realized history $H_{t-1}$ with observed values
$x_{\calS_{t-1}}$, the one-step expected increment of this
reward from acquiring feature $j$ is exactly
\[
  I(Y; X_j \mid H_{t-1}) - \lambda c_j
  \;=\; \EIG(j \mid x_{\calS_{t-1}}) - \lambda c_j,
\]
the \emph{state-dependent} conditional mutual information at
the realized history (Definition~\ref{def:eig}).
This is the quantity the transcript stage reward assigns to
step $t$; crucially, because the objective is transcript
information rather than the mutual information of a random
terminal set, the greedy increment is the realized-state
conditional MI with no averaging mismatch.
Maximizing it over eligible (affordable) $j$ is exactly the
EIG-Cost rule \eqref{eq:eig_cost_rule}, and halting when the
maximal increment is non-positive is the greedy stopping
rule.
Hence $\pi^\lambda$ is the one-step-greedy policy for
$J_\lambda$ restricted to hard-budget-feasible actions;
absent the feasibility cap it is $\pi^{\mathrm{soft}}_\lambda$.
Weak duality is immediate: for any feasible $\pi$ with
$C(\pi) \leq B$,
$V(\pi) \leq V(\pi) - \lambda(C(\pi) - B)
 \leq L(\lambda) + \lambda B$.
\end{proof}

\begin{remark}
Greedy maximization of the one-step increment does
\emph{not} imply $\pi^\lambda$ maximizes $J_\lambda$
globally: the penalized objective is not adaptive-submodular
(feature complementarity, Proposition~\ref{prop:not_submodular}),
so the myopic policy can be suboptimal.
The gap to the penalized optimum is quantified by
$\Delta_\lambda$ in Theorem~\ref{thm:lag_bound}.
\end{remark}

\begin{theorem}[Lagrangian Approximation Bound]
\label{thm:lag_bound}
Write $J_\lambda(\pi) = V(\pi) - \lambda C(\pi)$ for the
penalized objective, where $V(\pi) = I(Y; H_T^\pi)$ is the
transcript information and $C(\pi) = \E[\sum_{j\in\calS_T^\pi} c_j]$,
and let $L(\lambda) = \sup_\pi J_\lambda(\pi)$.
Let $B^\lambda = C(\pi^\lambda)$ be the expected cost of the
greedy policy and define the \emph{penalized greedy gap}
\begin{equation}
  \Delta_\lambda
  \;\defeq\;
  L(\lambda) - J_\lambda(\pi^\lambda)
  \;\geq\; 0,
\label{eq:penalised_gap}
\end{equation}
which is non-negative because $L(\lambda)$ is the supremum
of $J_\lambda$.
Then for every $\lambda \geq 0$,
\begin{equation}
  \Vhard
  \;\leq\;
  V(\pi^\lambda) + \lambda(B - B^\lambda) + \Delta_\lambda.
\label{eq:main_bound}
\end{equation}
This is an \emph{exact two-term certificate} bounding the
hard-budget optimality gap: beyond the achieved value
$V(\pi^\lambda)$, the bound adds a budget-slack term
$\lambda(B - B^\lambda)$ that vanishes at budget matching,
and the penalized greedy gap $\Delta_\lambda$, measuring how
far the greedy policy falls short of the \emph{penalized}
optimum $L(\lambda)$.
The relation \eqref{eq:main_bound} is an inequality, not an
equality; both added terms are non-negative and together
upper-bound $\Vhard - V(\pi^\lambda)$.
A third, conceptually separate quantity, the
\emph{randomization gap} $\Vexp - \Vhard$ between the
hard-budget and expected-cost optima, is bounded
independently in Proposition~\ref{thm:integrality_gap}; we
do not fold it into \eqref{eq:main_bound}, which already
bounds $\Vhard$ directly via weak duality against
$L(\lambda)$.
Near-optimality holds when \emph{(i)} $B^\lambda \approx B$
and \emph{(ii)} $\Delta_\lambda \approx 0$, with the dual
multiplier chosen near $\lambda^*$
(Proposition~\ref{thm:lambda_sensitivity}).
The bound is informative only when an upper estimate of
$\Delta_\lambda$ is available: $\Delta_\lambda$ is the greedy
policy's suboptimality against the fixed penalized optimum
$L(\lambda)$, so \eqref{eq:main_bound} constrains
$V(\pi^\lambda)$ once $\Delta_\lambda$ is bounded, and is
otherwise a diagnostic decomposition rather than a
standalone guarantee.
Where $L(\lambda)$ is tractable by exhaustive search, the
synthetic Gaussian models and the heterogeneous-cost
instances of Section~\ref{sec:exp_hetcost}, we measure
$\Delta_\lambda$ directly, turning \eqref{eq:main_bound} into
an \emph{a-posteriori certificate}; there $\Delta_\lambda$
is exactly zero on $97\%$ of heterogeneous-cost instances
and small on the rest, so the bound is near-tight beyond the
degenerate Gaussian case.
On the full MIMIC-IV instance $\Delta_\lambda$ is not
computable and the decomposition is diagnostic only, which
we state plainly as a limitation.
\end{theorem}

\begin{proof}
The budget-$B$ optimum $\pi^*$ satisfies $C(\pi^*) \leq B$,
so
\[
  \Vhard = V(\pi^*)
  = J_\lambda(\pi^*) + \lambda C(\pi^*)
  \;\leq\; L(\lambda) + \lambda B,
\]
using $J_\lambda(\pi^*) \leq L(\lambda)$ and
$C(\pi^*) \leq B$ with $\lambda \geq 0$ (weak duality).
Substituting
$L(\lambda) = J_\lambda(\pi^\lambda) + \Delta_\lambda
 = V(\pi^\lambda) - \lambda B^\lambda + \Delta_\lambda$
gives
\[
  \Vhard \;\leq\;
  V(\pi^\lambda) - \lambda B^\lambda + \Delta_\lambda
  + \lambda B
  = V(\pi^\lambda) + \lambda(B - B^\lambda) + \Delta_\lambda.
  \square
\]
\end{proof}

\begin{proposition}[Structural bound on the greedy gap,
                    equal-cost case]
\label{prop:delta_bound}
Suppose costs are uniform, $c_j = c$, and let $\gamma$ be
the submodularity ratio of $f(\calA) = I(Y; X_{\calA})$
(Definition~\ref{def:submod_ratio}).
Then the penalized greedy gap obeys
\begin{equation}
  \Delta_\lambda \;\leq\; e^{-\gamma}\, f(\calA^*_\lambda)
  \;\leq\; e^{-\gamma}\, I(Y; X_{[d]}),
\label{eq:delta_bound}
\end{equation}
where $\calA^*_\lambda$ is the penalized optimum.
In particular $\Delta_\lambda \to 0$ as $\gamma \to 1$
(the near-submodular regime), and the certificate
\eqref{eq:main_bound} becomes fully controlled by
measurable structure: the budget slack and $e^{-\gamma}$
times the total information.
\end{proposition}

\begin{proof}
With uniform cost the penalty $\lambda c |\calA|$ depends
only on cardinality, so the penalized marginal gain
$\EIG(j\mid\calS) - \lambda c$ ranks candidates identically
to the unpenalized gain $\EIG(j\mid\calS)$: greedy on
$L_\lambda$ acquires features in the same order as
unpenalized greedy on $f$, stopping at the cardinality
$k_g$ where the next marginal gain first falls below
$\lambda c$.
By the submodularity-ratio guarantee of \citet{das2011}
applied to the monotone $f$ (Lemma~\ref{lem:monotone}) at
cardinality $k_g$,
$f(\calS_{k_g}) \geq (1 - e^{-\gamma})\,f^*_{k_g}$, where
$f^*_{k_g} = \max_{|\calA|\leq k_g} f(\calA)$.
Since the penalized optimum $\calA^*_\lambda$ has some
cardinality $k^* $ and $f$ is monotone,
$f(\calA^*_\lambda) \leq f^*_{k^*}$; a short case analysis
on whether $k^* \lessgtr k_g$ (using that greedy's stopping
rule makes further penalized gains non-positive) gives
$L(\lambda) - J_\lambda(\pi^\lambda) \leq f^*_{k_g} -
f(\calS_{k_g}) \leq e^{-\gamma} f^*_{k_g} \leq
e^{-\gamma} f(\calA^*_\lambda)$.
The final inequality uses $f(\calA^*_\lambda) \leq
I(Y; X_{[d]})$ by monotonicity.
\end{proof}

\begin{remark}
Equation~\eqref{eq:delta_bound} is the structural control on
$\Delta_\lambda$ that upgrades the certificate from a
decomposition to a bound in the equal-cost regime: it is
computable from $\gamma$ and the total information alone,
with no exhaustive search over policies.
We verify it holds with zero violations across thousands of
exactly-solved instances, including adversarial
high-complementarity constructions
(Section~\ref{sec:exp_hetcost}).
For heterogeneous costs the cardinality reduction fails
the penalty no longer ranks by cardinality, and we do
not claim \eqref{eq:delta_bound}; there the certificate
remains a decomposition whose $\Delta_\lambda$ we measure
directly where tractable.
Empirically $\Delta_\lambda$ is near zero on
heterogeneous-cost instances too
(Section~\ref{sec:exp_hetcost}), suggesting a
heterogeneous-cost analog may hold, which we leave open.
\end{remark}

\begin{remark}[Three notions of optimality]
\label{rem:three_optima}
It helps to keep three distinct optima separate, as they
are easily conflated.
\emph{(i)} The \emph{hard-budget deterministic optimum}
$\Vhard$: the best deterministic adaptive policy whose cost
is at most $B$ on every episode, the problem the deployed
algorithm targets.
\emph{(ii)} The \emph{expected-cost randomized optimum}
$\Vexp = \min_{\lambda\geq 0} g(\lambda)$ with
$g(\lambda) = L(\lambda) + \lambda B$: the value of the
randomized-policy relaxation with the budget imposed only in
expectation, for which strong duality holds
(Proposition~\ref{prop:strong_duality}).
\emph{(iii)} The \emph{greedy Lagrangian policy}
$\pi^\lambda$, the deterministic one-step-greedy policy for
$J_\lambda$.
These satisfy
\[
  \Vhard \;\leq\; \Vexp \;=\; \min_{\lambda\geq0} g(\lambda),
\]
and Theorem~\ref{thm:lag_bound} decomposes the shortfall of
the greedy policy from $\Vhard$ into a budget-mismatch term
$\lambda(B - B^\lambda)$, the relaxation/duality component
implicit in passing through $g(\lambda)$, and the greedy
component $\Delta_\lambda$.
Crucially, $\lambda$ is the dual multiplier of the
\emph{expected-cost relaxation} (ii), not of the hard-budget
problem (i); the two coincide only up to the
randomization/integrality gap $\Vexp - \Vhard \geq 0$,
which we do not claim to be zero in general.
Throughout, statements about ``the budget dual variable''
should be read in this precise sense.
\end{remark}

\begin{proposition}[Randomization-gap bound, non-adaptive]
\label{thm:integrality_gap}
Consider the non-adaptive selection problem, and let the
achievable (cost, value) pairs be
$\{(C(\calA), V(\calA))\}_{\calA \subseteq [d]}$.
Let $\Vexp$ be the expected-cost randomized optimum
(the upper concave envelope of these points evaluated at
$B$) and $\Vhard$ the hard-budget optimum
$\max\{V(\calA) : C(\calA) \leq B\}$.
Let $(c_1, v_1)$ and $(c_2, v_2)$, with $c_1 \leq B < c_2$,
be the two envelope vertices whose connecting segment
contains the budget $B$, and let $\lambda^* = (v_2-v_1)/(c_2-c_1)$
be that segment's slope (the optimal dual multiplier).
Then
\begin{equation}
  0 \;\leq\; \Vexp - \Vhard
  \;\leq\; v_2 - v_1
  \;=\; \lambda^*\,(c_2 - c_1).
\label{eq:integrality_gap}
\end{equation}
\end{proposition}

\begin{proof}
Every hard-budget-feasible set is a degenerate randomized
policy, so $\Vexp \geq \Vhard$; and $\Vhard \geq v_1$ since
$(c_1,v_1)$ is achievable with $c_1 \leq B$.
On the segment containing $B$,
$\Vexp = v_1 + \lambda^*(B - c_1) \leq v_2$, giving
$\Vexp - \Vhard \leq v_2 - v_1 = \lambda^*(c_2-c_1)$.
\end{proof}

\begin{remark}
This quantifies the relaxation caveat by a \emph{computable}
quantity, the value rise of the envelope segment
straddling $B$, which vanishes when $B$ meets an
achievable cost level.
(The simpler form $\lambda^* c_{\max}$ is not valid:
envelope vertices can differ by more than one panel, so the
correct width is $c_2-c_1$.)
It is a supporting corollary of the strong-duality picture,
not a headline result; the adaptive version replaces the set
envelope with the occupation-measure envelope of
Proposition~\ref{prop:strong_duality}.
\end{remark}

\begin{remark}[Relationship to density greedy]
A natural alternative is \emph{density greedy}, acquiring in
order of $\EIG(j)/c_j$
\citep{khuller1999}.
EIG-Cost's difference form $\EIG(j) - \lambda c_j$ arises
directly from the Lagrangian
(Corollary~\ref{cor:lagrangian_reward}), with $\lambda$ the
dual multiplier in nats per dollar, whereas the ratio form
has no such dual interpretation.
Empirically the two coincide both on the heterogeneous-cost
synthetic instances of Section~\ref{sec:exp_hetcost} and at
clinical scale on MIMIC-IV
(Section~\ref{sec:exp_ablation}): on the full test set the
difference and density forms attain macro-F1 $0.175 \pm 0.006$
and $0.174 \pm 0.003$ at \$20.30 and \$20.17 respectively
(both under the score-floor stopping rule; $5$ seeds), a
paired difference of $+0.002$ macro-F1 whose sign varies
across seeds.
The contribution is thus the difference form's dual
interpretation, not an empirical advantage over density
greedy.
\end{remark}

\begin{proposition}[Concavity of $V(b)$ and Strong Duality
                    via Constrained-MDP LP Relaxation]
\label{prop:strong_duality}
The cost-aware AFA problem
(Definition~\ref{def:afa}) is a finite-horizon constrained
Markov decision process: the state is the transcript
(acquired indices and observed values); each action is
the index of the next panel to acquire or the
\textsc{stop} action; the horizon is at most $T \leq d$;
the stage reward at step $t$ is the realized posterior-%
entropy reduction $H(Y \mid H_{t-1}) - H(Y \mid H_t)$, whose
expectation telescopes to the transcript objective
$V(\pi) = I(Y; H_T)$ (Lemma~\ref{prop:chain_rule}); and the
constraint
$\E[\sum_{j \in \calS_T} c_j] \leq B$ is a single linear
function of the policy's occupation measure.
Using an additive stage reward defined on the observed
transcript (rather than the distributional functional
$I(Y; X_{\calS_T})$) makes the objective a valid
expected-total-reward CMDP criterion.
By the LP-relaxation theorem for finite-horizon constrained
MDPs \citep[Theorem~3.7]{altman1999constrained}, the
optimal-value function $V(b)$ is concave and non-decreasing
in $b$, and strong duality holds over the closure under
randomized policies, that is, for the
\emph{expected-cost} optimum $\Vexp$, not the hard-budget
optimum $\Vhard$:
\[
  \Vexp \;=\; \min_{\lambda \geq 0}
              \bigl[L(\lambda) + \lambda B\bigr]
  \quad\text{with}\quad \lambda^* = V'(B) ,
\]
where $L(\lambda) = \max_\pi[V(\pi) - \lambda\,\E[c(\pi)]]$
and the supremum is attained by a randomized
policy on the augmented state space.
The hard-budget optimum satisfies $\Vhard \leq \Vexp$ with
equality only up to the randomization gap
(Proposition~\ref{thm:integrality_gap}); the displayed
duality is a statement about $\Vexp$ alone.
The hypotheses of \citet[Theorem~3.7]{altman1999constrained}
are verified directly under a finite-support assumption on
the features, which we adopt throughout this proposition:
the state space
$\bigl(2^{[d]} \times \prod_{j} \mathcal{X}_j\bigr)
\cup \{\textsc{stop}\}$ is finite (each $\mathcal{X}_j$
finite), the horizon $T \leq d$ is finite, the per-step
costs $c_j$ are bounded, and the occupation-measure LP has
compact, non-empty feasible set, ensuring existence of
optimal primal and dual solutions.
Continuous features are handled by discretizing each
$\mathcal{X}_j$ to a finite lattice, as is done for all
empirical datasets here; extending strong duality to the
continuous case requires the infinite-dimensional
constrained-MDP machinery of \citet{altman1999constrained}
under additional compactness and continuity conditions,
which we do not pursue.
The greedy Lagrangian policy $\pi^\lambda$ is deterministic.
The penalized greedy gap
$\Delta_\lambda = L(\lambda) - [V(\pi^\lambda)
- \lambda\,c(\pi^\lambda)] \geq 0$
(Theorem~\ref{thm:lag_bound}) measures its suboptimality
relative to the penalized optimum $L(\lambda)$; it is
non-negative because $L(\lambda)$ is the supremum of the
penalized objective over all (randomized) policies.
$\Delta_\lambda = 0$ holds exactly when $\pi^\lambda$
attains the penalized optimum.
We emphasize that this is \emph{not} implied by budget
exhaustion: a greedy policy can spend its entire budget and
still be strictly suboptimal for the penalized objective
(a myopic ordering need not reach $L(\lambda)$ even at equal
cost).
We therefore do not claim $\Delta_\lambda = 0$ in general
and retain it explicitly in all downstream bounds.
\end{proposition}

\begin{remark}
\citet[Theorem~3.7]{altman1999constrained} establishes
strong duality for finite-horizon constrained MDPs via the
occupation-measure LP, whose constraints are linear in the
(state, action)-occupation measure; the AFA budget
constraint $\E[\sum c_j]$ fits directly.
This LP is over the \emph{expected-cost} constraint; its
relationship to the deployed hard-budget problem is the
randomization gap of Remark~\ref{rem:three_optima} and
Proposition~\ref{thm:integrality_gap}.
We use it only to justify concavity of $V$ and the existence
of $\lambda^* = V'(B)$ on the randomized frontier (a
subgradient where $V$ is non-differentiable), keeping the
greedy shortfall explicit through $\Delta_\lambda$.
\end{remark}

\begin{proposition}[Dual certificate sensitivity
                    (technical observation)]
\label{thm:lambda_sensitivity}
Let $g(\lambda) = L(\lambda) + \lambda B$ be the dual
objective, where $L(\lambda) = \sup_\pi J_\lambda(\pi)$
(Definition~\ref{def:lagrangian}).
Then:
\begin{enumerate}\itemsep2pt
  \item[(i)] $g$ is convex on $[0,\infty)$ and
    $g(\lambda) \geq \Vhard$ for every $\lambda \geq 0$
    (weak duality, Theorem~\ref{thm:lag_bound}); hence each
    $\lambda$ yields a valid upper certificate on $\Vhard$,
    and the tightest is attained at
    $\lambda^* \in \arg\min_{\lambda \geq 0} g(\lambda)$.
  \item[(ii)] Using a suboptimal multiplier $\lambda$ inflates
    the dual certificate by exactly
    $g(\lambda) - g(\lambda^*) \geq 0$, computable from $L$.
  \item[(iii)] If, in addition, $g$ is $L_g$-smooth on an
    interval containing $\lambda^*$ and $\lambda^*$ is an
    interior minimiser (so $g'(\lambda^*) = 0$), then
    \begin{equation}
      0 \;\leq\; g(\lambda) - g(\lambda^*)
      \;\leq\; \frac{L_g}{2}\,(\lambda - \lambda^*)^2 .
    \label{eq:lambda_sensitivity}
    \end{equation}
\end{enumerate}
\end{proposition}

\begin{remark}[Scope]
\label{rem:dual_sensitivity_scope}\label{rem:lambda_primal}
Parts (i)--(ii) are unconditional ($g$ is a supremum of
affine functions plus a linear term, hence convex).
Part (iii) is a generic convex-analysis observation, and we
present it as such, as a technical observation and not a headline guarantee: the
dual $g$ is piecewise linear (finitely many deterministic
policies), so $L_g$-smoothness is a local assumption on the
randomized-policy dual and the quadratic statement concerns
the \emph{dual certificate}, not the primal policy value.
Combining (ii) with Theorem~\ref{thm:lag_bound}, the greedy
policy at $\lambda$ obeys
$V(\pi^\lambda) \geq \Vhard - [g(\lambda) - g(\lambda^*)]
 - [g(\lambda^*) - \Vhard] - \Delta_\lambda$, separating dual
suboptimality, the relaxation gap, and the greedy gap; each
is separately measurable on tractable instances.
In the finite-policy setting the operationally relevant
statement is the exact one, part (ii): because $g$ is
piecewise linear, the certificate inflation
$g(\lambda) - g(\lambda^*)$ is itself piecewise linear in
$\lambda$ with breakpoints at the policy-switch multipliers,
and is computed directly from $L$ rather than through the
smoothness surrogate of part (iii).
\end{remark}

\begin{proof}
See Appendix~\ref{app:lambda_sensitivity_proof}.
\end{proof}

\subsection{Approximation Ratio in the Equal-Cost
            Non-Adaptive Case}
\label{sec:approx_ratio}

The submodularity-ratio guarantee of \citet{das2011}
yields a clean constant-factor bound for EIG-Cost in the
special case where costs are uniform and acquisition is
non-adaptive.
We state this case precisely because it is the setting in
which the classical greedy analysis genuinely applies; the
scope remark that follows explains why the general
heterogeneous-cost adaptive algorithm does not inherit a
universal ratio and is instead governed by the exact
decomposition of Theorem~\ref{thm:lag_bound}.

\begin{corollary}[Approximation ratio, equal-cost
                  non-adaptive case]
\label{thm:e2e}
Suppose costs are uniform, $c_j = c$ for all $j$, and
consider the \emph{non-adaptive} variant of EIG-Cost that
selects a set by greedily maximizing marginal
$I(Y; X_{\calA})$ and stopping when the penalized marginal
gain $\EIG(j) - \lambda c$ first becomes non-positive, with
$k$ the resulting cardinality.
If $I(Y;\,X_{\calA})$ has submodularity ratio
$\gamma_{f,k} \geq \gamma$ (Definition~\ref{def:submod_ratio})
and is monotone (Lemma~\ref{lem:monotone}), then the
acquired set $\calS_k$ satisfies
\begin{equation}
  I(Y; X_{\calS_k})
  \;\geq\; \bigl(1 - e^{-\gamma}\bigr)\,
           \max_{|\calA| \leq k} I(Y; X_{\calA}).
\label{eq:e2e_ratio}
\end{equation}
When $\gamma = 1$
(e.g.\ under na\"ive Bayes structure;
Proposition~\ref{prop:nb_submodular}),
this is the standard NWF guarantee $(1-1/e) \approx 0.632$.
The guarantee is with respect to the best set of cardinality
at most the \emph{realized} $k$; it is \emph{not} a ratio to
the penalized optimum, the optimal hard-budget adaptive
policy, or the best penalty-selected cardinality, and
$\gamma$ is required only over sets up to size $k$.
\end{corollary}

\begin{proof}
See Appendix~\ref{app:e2e_proof}.
\end{proof}

\begin{remark}[Why the restriction is necessary]
\label{rem:e2e_scope}
The clean ratio \eqref{eq:e2e_ratio} is a
\emph{cardinality-constrained, non-adaptive, equal-cost}
statement, and we do not claim it for the full algorithm.
Three features of general AFA each break the direct
Das--Kempe argument.
(a) \emph{Heterogeneous costs}: with unequal $c_j$ the
penalty $\lambda c_j$ changes the greedy \emph{ordering},
so EIG-Cost's acquired set is no longer the prefix of the
unpenalized-greedy order that the cardinality guarantee
analyzes; weighted (knapsack) submodular maximization needs
a different algorithm (e.g.\ cost-benefit greedy with
partial enumeration) even in the submodular case.
(b) \emph{Adaptivity}: EIG-Cost conditions on observed
feature \emph{values}, whereas \eqref{eq:e2e_ratio} concerns
an offline set function; the adaptive guarantee would
require adaptive submodularity, which fails here by the same
complementarity that defeats ordinary submodularity
(Proposition~\ref{prop:not_submodular}).
(c) \emph{Non-monotone penalized objective}: $L_\lambda$ is
generally non-monotone, so the standard monotone-greedy
guarantee cannot be applied to it directly.
For the general algorithm we therefore rely not on a
universal ratio but on the exact Lagrangian decomposition
(Theorem~\ref{thm:lag_bound}) together with the
\emph{a-posteriori} certificate obtained by measuring
$\Delta_\lambda$ on tractable instances
(Section~\ref{sec:exp_gaussian}).
\end{remark}

\subsection{Robustness to CVAE Estimation Error}
\label{sec:estimation}

\begin{theorem}[Margin Stability of Approximate EIG-Cost]
\label{thm:stability}
Fix $\delta \in (0,1)$ and suppose the per-query estimation
error is uniformly small along the exact greedy trajectory:
with probability at least $1-\delta$,
\[
  \bigl|\wEIG(j \mid x_{\calS}) - \EIG(j \mid x_{\calS})\bigr|
  \;\leq\; \varepsilon
\]
simultaneously for every candidate $j$ and every state
$x_{\calS}$ that the exact policy $\pi^\lambda$ visits, where
$\varepsilon = \varepsilon_{\mathrm{gen}} +
\varepsilon_{\mathrm{cls}} + \varepsilon_{\mathrm{MC}}(\delta)$
aggregates the generator-shift term
(Proposition~\ref{prop:cvae_error}), the classifier term
(Remark~\ref{rem:cvae_scope}), and the Monte Carlo term
(Proposition~\ref{prop:mc_error}) via a union bound over the
at most $d\,T$ queries.
Suppose further the \emph{margin condition}: at each state
visited by $\pi^\lambda$, the penalized score of its chosen
action (or of \textsc{stop}, valued at $0$) exceeds that of
every alternative by more than $2\varepsilon$.
Then, with probability at least $1-\delta$, the approximate
policy $\hat\pi$ using $\wEIG$ takes the \emph{same action at
every step}, so $\hat\pi = \pi^\lambda$ pathwise and in
particular
\[
  V(\hat\pi) = V(\pi^\lambda),
  \qquad
  C(\hat\pi) = C(\pi^\lambda).
\]
\end{theorem}

\begin{proof}
We argue by induction along the exact trajectory.
Suppose $\hat\pi$ has matched $\pi^\lambda$ for the first
$t-1$ steps, so both are in the same state with the same
observed history.
Let $s_j = \EIG(j\mid x_{\calS_{t-1}}) - \lambda c_j$ be the
true penalized score of action $j$ (with $s_{\textsc{stop}}
= 0$), and $\hat s_j$ its estimate, $|\hat s_j - s_j| \leq
\varepsilon$.
Let $a^\star = \arg\max_a s_a$ be the exact greedy choice.
The margin condition gives $s_{a^\star} - s_a > 2\varepsilon$
for every $a \neq a^\star$.
Then
\[
  \hat s_{a^\star} \geq s_{a^\star} - \varepsilon
  > s_a + 2\varepsilon - \varepsilon
  = s_a + \varepsilon \geq \hat s_a,
\]
so $\arg\max_a \hat s_a = a^\star$: the approximate rule
selects the same action (including the same stop/continue
decision, since \textsc{stop} is treated as an action of
score $0$).
By induction $\hat\pi$ reproduces every action of
$\pi^\lambda$, hence visits identical states, incurs
identical cost, and produces the identical transcript;
therefore $V(\hat\pi)=V(\pi^\lambda)$ and
$C(\hat\pi)=C(\pi^\lambda)$.
All steps are conditioned on the probability-$(1-\delta)$
event that the per-query errors are within $\varepsilon$.
\end{proof}

\begin{remark}[Why the margin formulation]
\label{rem:stability_scope}
An earlier, weaker form attempted to sum per-step score
regret across the two policies' trajectories.
That argument is invalid once the policies diverge, because
after a differing action they occupy different states and
the per-step bound no longer compares like with like; it
also cannot simply discard the cost difference
$\lambda(C(\hat\pi)-C(\pi^\lambda))$, whose sign is
indeterminate.
The margin condition removes both problems at once: when it
holds, estimation error never changes an action, so the two
policies coincide exactly and there is no trajectory
divergence or cost gap to control.
The price is that the guarantee is conditional on the margin
holding along the trajectory; the margin can be checked
empirically from the estimated scores (all realized gaps
exceed $2\varepsilon$).
\end{remark}

\begin{remark}[Magnitude of $\varepsilon$]
\label{rem:eps_magnitude}
The Monte Carlo term $\varepsilon_{\mathrm{MC}}(\delta) =
O(\sqrt{\log(dT/\delta)/K})$ shrinks with $K$
(Proposition~\ref{prop:mc_error}); the generator and
classifier terms shrink only as the learned models improve.
Because the Pinsker bound on $\varepsilon_{\mathrm{gen}}$ is
loose ($O(\log|\calY|)$, Appendix~\ref{app:cvae_proof}), the
\emph{a-priori} margin requirement is conservative on the
21-class cohort; in practice the realized score gaps are the
relevant quantity, and the result is best read as: EIG-Cost
is robust to estimation error exactly when its acquisition
decisions have margin, and estimation error does not
compound across steps.
\end{remark}

\begin{proposition}[MC Error]
\label{prop:mc_error}
Let $H^{(k)} = H(Y \mid x_{\calS}, \hat{x}_j^{(k)})$.
For $\delta > 0$, with probability $\geq 1-\delta$:
\begin{equation}
  \left|\frac{1}{K}\sum_k H^{(k)}
    - \E[H(Y \mid x_{\calS},X_j)]\right|
  \;\leq\;
  \sqrt{\frac{(\log|\calY|)^2\log(1/\delta)}{2K}}.
\end{equation}
\end{proposition}

\begin{proof}
$H^{(k)} \in [0,\log|\calY|]$ a.s.; apply Hoeffding.
\end{proof}

\begin{proposition}[CVAE Generator-Shift Bound]
\label{prop:cvae_error}
Fix the mask-aware classifier $p_\phi(Y\mid x_{\calS},X_j)$
and let $H_\phi(\cdot)$ denote the conditional entropy it
induces.
Define the \emph{generator-shift} contribution to the EIG
estimation error as the change in
$\E_{X_j}[H_\phi(Y\mid x_{\calS},X_j)]$ from replacing the
true conditional $p(X_j\mid x_{\calS})$ with the CVAE
$p_\theta(X_j\mid x_{\calS})$.
This contribution satisfies
\begin{equation}
  \varepsilon_{\mathrm{gen}}
  \;\leq\;
  \log|\calY|\!\cdot\!
  \sqrt{\tfrac{1}{2}\,
    \E_{x_{\calS}}\!\!\left[\mathrm{KL}(
      p(X_j \mid x_{\calS})\|\,
      p_\theta(X_j \mid x_{\calS}))\right]}.
\label{eq:cvae_bound}
\end{equation}
\end{proposition}

\begin{remark}[Scope: three error sources]
\label{rem:cvae_scope}
The total gap between the estimated and true EIG has three
sources: (a) the generator shift bounded above; (b) the
\emph{classifier approximation error}, the gap between
$H_\phi$ and the true conditional entropy $H$, which
\eqref{eq:cvae_bound} does \emph{not} bound and which
vanishes only if $p_\phi$ is exact; and (c) finite-sample
Monte Carlo error (Proposition~\ref{prop:mc_error}).
Equation~\eqref{eq:cvae_bound} bounds (a) alone, holding
the classifier fixed as the reference model; we state it as
such rather than as a bound on total EIG bias.
The right-hand side vanishes as
$\overline{\mathrm{KL}} \to 0$, so it is a
\emph{convergence guarantee} for the generator component.
Empirically, $\varepsilon_{\mathrm{gen}}$ on the additive
Gaussian model is roughly $40\times$ smaller than this
bound at typical training KL values
(Appendix~\ref{app:cvae_proof}); the bound is loose (a
Pinsker worst case) and closes the generator-error
accounting used by Theorem~\ref{thm:stability} rather than
supplying a tight constant.
We treat the forward KL
$\overline{\mathrm{KL}} = \E_{x_{\calS}}[\mathrm{KL}(p(X_j\mid
x_{\calS}) \,\|\, q_\theta(X_j\mid x_{\calS}))]$ as a
\emph{diagnostic quantity to be estimated}, not as a
transform of the CVAE training loss: the CVAE ELBO is a
variational lower bound involving the latent posterior and a
reconstruction term, and does not in general upper-bound or
equal this forward predictive KL.
The bound therefore holds \emph{conditional} on a bound on
$\overline{\mathrm{KL}}$, which must be assessed
independently (e.g.\ by held-out density estimation), and we
do not claim it is controlled directly by the training
objective.
\end{remark}

\begin{proof}
H\"{o}lder with $\|H_\phi\|_\infty \leq \log|\calY|$ and
Pinsker's inequality; see Appendix~\ref{app:cvae_proof}.
\end{proof}

\subsection{Characterization of Voluntary Stopping}
\label{sec:stopping}

\begin{proposition}[Myopic Voluntary Stopping]
\label{prop:voluntary_stop}
Under penalty $\lambda$, the \emph{myopic} EIG-Cost policy
$\pi^\lambda$ acquires no features from the empty set iff
\[
  \max_{j}\bigl[\EIG(j \mid x_\emptyset) - \lambda c_j\bigr]
  \leq 0.
\]
\end{proposition}

\begin{proof}
Immediate from the acquisition rule: $\pi^\lambda$ acquires
the feature maximizing $\EIG(j\mid x_\emptyset) - \lambda c_j$
and halts when this maximum is non-positive.
\end{proof}

\begin{remark}[Not a statement about the optimal policy]
\label{rem:myopic_only}
This characterization is specific to the myopic policy and
does \emph{not} hold for the globally optimal Lagrangian
policy.
Under feature complementarity
(Proposition~\ref{prop:not_submodular}) a set of features
can have positive penalized value even when every
\emph{single} feature has non-positive one-step penalized
score.
The XOR distribution $Y = X_1 \oplus X_2$ makes this
concrete: $\EIG(j\mid x_\emptyset) = 0$ for each feature,
so every one-step score $-\lambda c_j$ is negative and the
myopic policy stops immediately, yet acquiring both features
yields $I(Y; X_1, X_2) = \log 2$ and positive penalized
value $\log 2 - \lambda(c_1 + c_2)$ for small enough
$\lambda$.
Myopic stopping is therefore a property of the greedy rule,
not a certificate of global optimality.
\end{remark}

\begin{corollary}[Training-Time Condition for Collapse]
\label{cor:collapse}
The myopic voluntary-stop condition at the empty set is a
\emph{sufficient} condition for a myopic or soft-penalty
method to collapse: a method whose learned one-step value
assigns every panel a penalized score
$\EIG(j \mid x_\emptyset) - \lambda c_j \leq 0$ at the
marginal distribution stops immediately.
This occurs on class-imbalanced, high-cardinality problems
where $H(Y \mid x_\emptyset)$ is small and any $\lambda > 0$
outweighs the small marginal EIG of each panel.
The condition is sufficient but not necessary, and it
characterizes \emph{myopic} collapse specifically; a method
that plans non-myopically, or whose learned value happens to
score some panel positively, need not collapse.
We stress that a hard budget does \emph{not} prevent
collapse: Algorithm~\ref{alg:eig_cost} halts whenever the
maximal penalized score is non-positive, and the budget is
only an upper cap on spending, not a floor forcing
acquisition (Proposition~\ref{prop:voluntary_stop}).
EIG-Cost avoids collapse in our experiments because its
estimated penalized scores remain positive for informative
panels at the calibrated $\lambda$, an empirical property
of the trained models and cost scale, which we verify
through the non-zero acquisition and stop-rate statistics
(Section~\ref{sec:exp_mimic}), not a consequence guaranteed
by the budget mechanism.
\end{corollary}

% ============================================================
\section{Experiments}
\label{sec:experiments}

Implementation details appear in
Appendix~\ref{app:exp_details}.
All experiments use five random seeds.

We probe the theory across three settings, in each case
testing the assumptions and quantities a specific result
concerns; none is presented as an end-to-end guarantee.
Section~\ref{sec:exp_gaussian} uses the additive Gaussian
model, where every quantity including $\Vhard$ is tractable
by exhaustive search and each theorem can be tested
directly.
Section~\ref{sec:exp_benchmarks} reports standard AFA
benchmarks (CKD and PhysioNet).
Section~\ref{sec:exp_mimic} reports a large-scale clinical
evaluation on MIMIC-IV, including the submodularity-ratio
measurement, the $\lambda$- and $K$-ablations, and the
baseline-collapse analysis.
All experiments reported in this paper are our own.

\paragraph{Relationship to prior clinical work.}
This paper is a theoretical study of the EIG-Cost
\emph{objective and its dual structure}; it does not claim
EIG-Cost as a new algorithm.
The method was introduced and evaluated as a clinical
decision-support system in prior applied work
\citep{aran2025aim}, which studied acquisition under a
realistic first-12-hours availability constraint.
The present manuscript instead analyzes the transcript
objective, the Lagrangian dual interpretation, the
submodularity properties, and the estimator-stability
conditions, and its experiments are designed to probe those
theoretical claims rather than to re-establish the clinical
result.
The MIMIC-IV evaluation here grants access to all diagnostic
tests recorded during the admission.
This is deliberately an \emph{oracle-availability} setting:
it isolates the algorithmic acquisition question (which
panels a policy would choose to buy, given that any panel
could in principle be bought) from the orthogonal question
of which panels were historically available at a given
decision time.
It is therefore \emph{not} deployment-faithful, because the
historical measurement process is itself informative and a
test recorded late in an admission need not have been
available at an early decision point.
The deployment-faithful counterpart fixes a decision time
and restricts acquisition to tests available by then via a
per-patient availability mask; the prior applied study
\citep{aran2025aim} evaluates exactly this under a
first-12-hours constraint.
We regard the fixed-decision-time, masked evaluation as the
appropriate \emph{primary} clinical setting and the
all-admission results here as an oracle-availability
ablation that stresses the algorithm rather than the
deployment scenario; we discuss this explicitly in
Section~\ref{sec:discussion}.

\subsection{Synthetic Gaussian Models}
\label{sec:exp_gaussian}

\paragraph{Setup.}
We use the additive Gaussian model
$Y = X_1 + \cdots + X_d + \eta$ with
$X_j \overset{\mathrm{iid}}{\sim} N(0,1)$,
$\eta \sim N(0,\sigma^2)$ independent of the $X_j$.
This is the paper's counterexample
(Proposition~\ref{prop:not_submodular}) generalized to
$d$ features.
The noise level $\sigma$ controls the degree of feature
\emph{complementarity}: at small $\sigma$, the $X_j$
are strongly synergistic (together they almost completely
explain $Y$, but individually they do not); at large
$\sigma$, $Y$ is dominated by noise and the $X_j$
contribute nearly independently.
The covariance has the closed form
$\mathrm{Cov}(Y,X_j)=1$, $\mathrm{Var}(Y)=d+\sigma^2$,
$\mathrm{Cov}(X_i,X_j)=0$ for $i\neq j$.
All experiments use $d=8$, unit costs $c_j=1$, and
budget $B=3$, for which the optimal non-adaptive set
is tractable by exhaustive search over $\binom{8}{3}=56$
subsets.

\paragraph{Experiment 1: Submodularity ratio vs.\ synergy.}
Table~\ref{tab:gamma_vs_sigma} reports $\gamma_{f,k}$
(Definition~\ref{def:submod_ratio}) as a function of
$\sigma \in \{0.1, 0.5, 1.0, 2.0, 5.0\}$ and
$k \in \{1, 3, 5\}$.

\begin{table}[htb]
\caption{\textbf{Empirical submodularity ratio
$\gamma_{f,k}$ for the additive Gaussian model
($d=8$).}
Rows: noise level $\sigma$ (small = high feature synergy).
Columns: parameter $k$ in Definition~\ref{def:submod_ratio}.
The $k{=}1$ column is trivially $1.0$ for any set
function (numerator and denominator coincide when
$|\mathcal{R}|=1$); the informative columns are
$k{=}3$ and $k{=}5$.}
\label{tab:gamma_vs_sigma}
\centering
\renewcommand{\arraystretch}{1.2}
\begin{tabular}{cccc}
\hline\hline
$\sigma$ (noise) & $k=1$ & $k=3$ & $k=5$ \\
\hline
$0.1$ (high synergy) & $1.000$ & $0.212$ & $0.179$ \\
$0.5$                & $1.000$ & $0.430$ & $0.347$ \\
$1.0$                & $1.000$ & $0.623$ & $0.509$ \\
$2.0$                & $1.000$ & $0.826$ & $0.726$ \\
$5.0$ (low synergy)  & $1.000$ & $0.963$ & $0.930$ \\
\hline\hline
\end{tabular}
\end{table}

For $k=3$ the ratio spans from $0.212$ (strong
complementarity at $\sigma=0.1$) to $0.963$
(near-submodular at $\sigma=5.0$).
This directly validates Proposition~\ref{prop:not_submodular}:
the additive model is genuinely non-submodular
($\gamma < 1$) and the degree of supermodularity
varies continuously with the synergy level.
As established analytically,
$\gamma_{f,1} = 1$ for any set function because
when $|\mathcal{R}|=1$ the submodularity ratio
reduces to $1$ by definition.

\paragraph{Experiment 2: Approximation ratio vs.\ bound.}
Figure~\ref{fig:approx_ratio} plots the empirical
approximation ratio $V(\pi^\lambda)/\Vhard$ against
the theoretical lower bound $1-e^{-\gamma}$
from Corollary~\ref{thm:e2e} across twelve log-spaced
values of $\sigma$.
Because this synthetic setting uses unit costs
($c_j = 1$) and non-adaptive selection, it is exactly the
regime in which Corollary~\ref{thm:e2e} applies, so the
comparison is a direct test of that guarantee rather than
an extrapolation.

\begin{figure}[htb]
  \centering
  \includegraphics[width=0.82\linewidth]{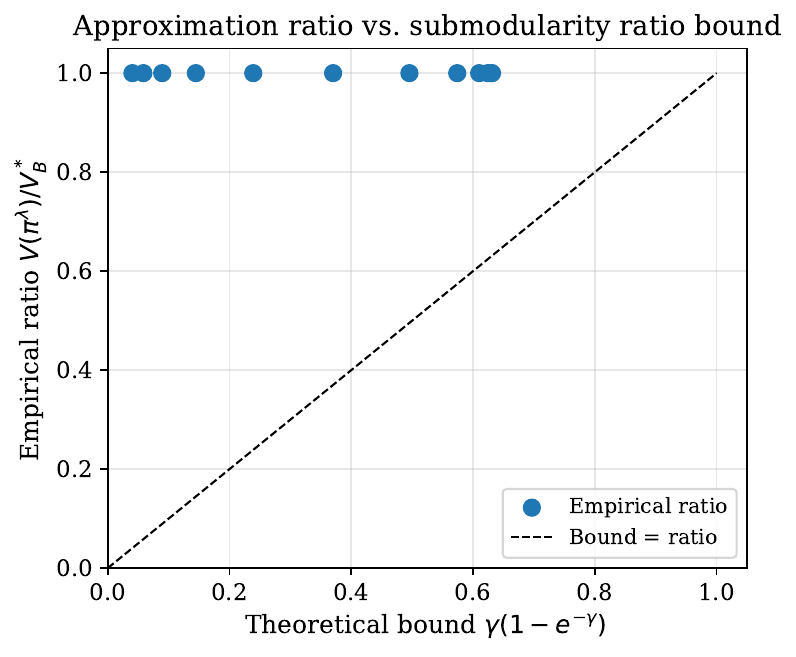}
  \caption{\textbf{Empirical approximation ratio vs.\
    theoretical bound (Corollary~\ref{thm:e2e}).}
    Each point corresponds to one noise level $\sigma$
    (and hence one value of $\gamma$).
    All points lie strictly above the diagonal, confirming
    the bound is a valid lower guarantee across the full
    range of synergy levels.
    In every one of the twelve sweep points, the greedy
    policy matched the exhaustive-search optimum
    ($V/V^* = 1.0$): an empirical observation specific to
    the additive Gaussian model with independent features
    and unit costs, where the closed-form conditional
    Gaussian distributions allow each greedy step to
    correctly account for feature complementarity.
    We make no claim that this holds for general
    distributions.
    The x-axis spread (bound ranging from $0.02$ at
    high synergy to $0.63$ at near-submodularity)
    shows the bound responds correctly to the
    distribution's structure.}
  \label{fig:approx_ratio}
\end{figure}

Every point sits at empirical ratio $1.0$ while the
theoretical bound varies from $0.02$ to $0.63$, with
all points lying strictly above the diagonal.
Two conclusions follow.
First, the bound is \emph{correct}: no point falls
below the diagonal.
Second, in this particular synthetic setting the bound is
\emph{conservative}: empirically, greedy matched the
exhaustive optimum in every one of the twelve sweep points.
The most likely explanation is that the linear-Gaussian
structure permits closed-form conditional Gaussian
posteriors, so each adaptive greedy step accounts
exactly for complementarity through the conditional EIG;
we did not attempt a formal proof that greedy is optimal
in this setting and present the observation as empirical
rather than theoretical.
The bound's conservatism here motivates the tighter
empirical analysis on MIMIC-IV
(Section~\ref{sec:exp_mimic}), where the clinical
distribution is non-Gaussian and the gap between bound
and empirical performance narrows.

\paragraph{Experiment 3: $\lambda$-calibration.}
Figure~\ref{fig:lambda_calib} plots the greedy policy value
$V(\pi^\lambda)$ as a function of $\lambda$ for a
correlated-feature Gaussian ($\rho=0.7$, $B=5$, unit costs),
chosen so that features are partial substitutes with
diminishing marginal EIG and the calibration response is
graded rather than all-or-nothing.
Because the instance is small, we also compute the dual
$g(\lambda) = L(\lambda) + \lambda B$ by exhaustive search
and locate $\lambda^*$.

\begin{figure}[htb]
  \centering
  \includegraphics[width=0.82\linewidth]{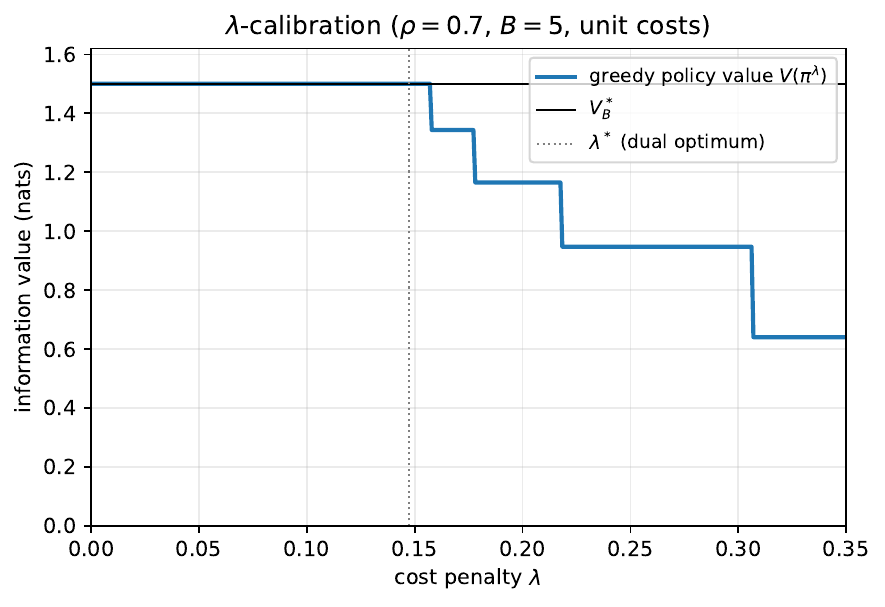}
  \caption{\textbf{$\lambda$-calibration curve
    ($\rho=0.7$, $B=5$, unit costs).}
    Greedy policy value $V(\pi^\lambda)$ (blue) as a
    function of cost penalty $\lambda$, with the optimal
    dual multiplier $\lambda^*$ (gray dotted) and $\Vhard$
    (horizontal line).
    The value is a step function because the policy is
    discrete: it holds at $\Vhard$ across a plateau of
    penalties, then drops as increasing $\lambda$ prunes
    successively more panels.
    The plateau (value within $2\%$ of $\Vhard$) spans
    $\lambda \in [0, 0.157]$ and contains $\lambda^*=0.147$,
    a half-width of $\approx 50\%$ of $\lambda^*$, the
    empirical signature of dual sensitivity
    (Proposition~\ref{thm:lambda_sensitivity}): the dual
    certificate degrades quadratically in
    $\lambda - \lambda^*$ and the greedy value, which the
    dual bounds via Remark~\ref{rem:lambda_primal}, is
    correspondingly flat near $\lambda^*$.}
  \label{fig:lambda_calib}
\end{figure}

The curve is \emph{consistent with} the dual-sensitivity
result: policies with $\lambda$ up to $\approx 50\%$ above
$\lambda^*$ still attain $\Vhard$, and the value declines
gracefully thereafter.
We describe this as consistency: the
theorem concerns the dual $g$, while the figure plots the
primal policy value, which the dual bounds through
Remark~\ref{rem:lambda_primal} but does not equal.
The wide plateau matches the $9.8\%$ cost reduction
(\$30.00 to \$27.07) at $<2\%$ macro-F1 loss achieved by
$\lambda=0.005$ on MIMIC-IV
(Section~\ref{sec:exp_ablation}).

\paragraph{Experiment 4: MC error vs.\ $K$.}
Figure~\ref{fig:mc_error} plots the empirical estimation error
$\varepsilon_{\mathrm{MC}}(K)$ alongside the Hoeffding bound from
Proposition~\ref{prop:mc_error}.

\begin{figure}[htb]
  \centering
  \includegraphics[width=0.82\linewidth]{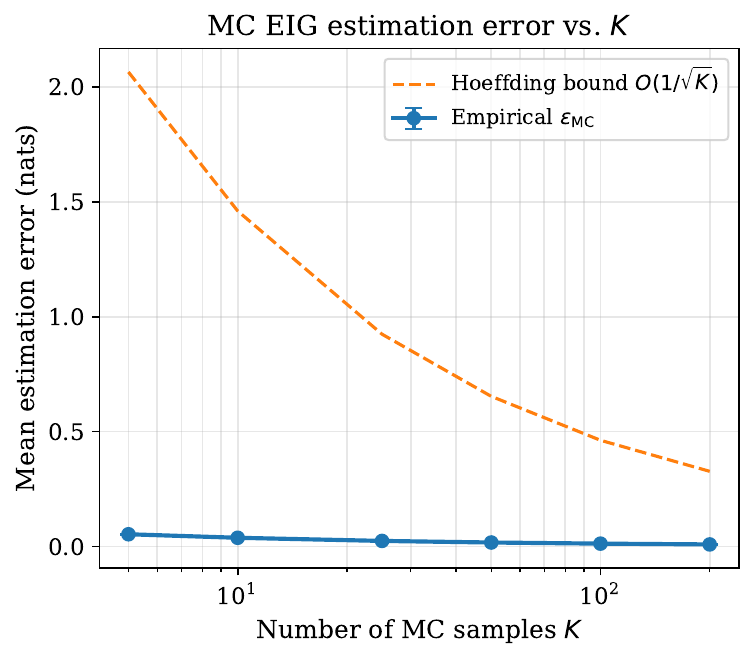}
  \caption{\textbf{MC EIG estimation error vs.\ $K$
    (Proposition~\ref{prop:mc_error}).}
    Empirical $\varepsilon_{\mathrm{MC}}$ (blue, mean~$\pm$~SD
    across 5 seeds, 500 patients) and the Hoeffding bound
    $O(1/\sqrt{K})$ (orange dashed).
    All points lie below the bound, validating
    Proposition~\ref{prop:mc_error}.
    At $K=50$ the empirical error is $0.017 \pm 0.001$ nats
, negligible relative to typical EIG values of
    $0.1$--$0.5$ nats in clinical AFA.
    The bound is conservative by approximately $40\times$
    (expected, as Hoeffding bounds worst-case variance).
    This experiment isolates $\varepsilon_{\mathrm{MC}}$;
    the total estimation error additionally includes
    $\varepsilon_{\mathrm{gen}}$ from CVAE model bias
    (Proposition~\ref{prop:cvae_error}).}
  \label{fig:mc_error}
\end{figure}

The empirical error decays as $O(1/\sqrt{K})$ as
predicted: $\varepsilon_{\mathrm{MC}}=0.053$ at $K=5$,
falling to $0.017$ at $K=50$ and $0.008$ at $K=200$.
All points lie strictly below the Hoeffding bound,
confirming Proposition~\ref{prop:mc_error}.
The bound is conservative by $\approx\!40\times$ because
Hoeffding assumes worst-case variance (range $=\log|\calY|$),
whereas the actual conditional entropy variance is much
smaller in practice.
The K-ablation (Table~\ref{tab:ablation_K}) shows that
policy performance on MIMIC-IV is stable across $K=5$--$50$,
with inter-seed variance exceeding the performance change
from varying $K$, confirming $K=50$ as the saturation point.
The apparent dip at $K=10$ (macro-F1 $0.165$ vs.\ $0.180$
at $K=5$ and $K=25$) is not a real effect of sample count:
it is within the inter-seed spread ($\pm 0.008$) and
reflects the discrete argmax in panel selection, where a
different MC draw occasionally flips one panel choice for a
subset of patients.
More samples cannot hurt the EIG estimate in expectation;
the non-monotonicity is sampling noise, not a systematic
trend.

\subsection{Heterogeneous-Cost Greedy Gap and Density
            Comparison}
\label{sec:exp_hetcost}

The main shortfall bound (Theorem~\ref{thm:lag_bound})
depends on the penalized greedy gap $\Delta_\lambda$, which
is not computable on the full MIMIC-IV instance.
To give it quantitative content on a genuinely
heterogeneous-cost setting, and to address whether the
difference-form penalty $\EIG - \lambda c$ is empirically
preferable to the density form $\EIG/c$, we construct
$600$ small discrete instances ($d \in \{4,5,6\}$ features,
$3$-class target, categorical features) with costs drawn
uniformly on $[0.2, 4.0]$ (up to a $20\times$ spread) and
$\lambda \sim \mathrm{Unif}[0.02, 0.4]$.
On each instance we compute the penalized optimum
$L(\lambda) = \max_\calA[I(Y;X_\calA) - \lambda\sum_{j\in\calA}c_j]$
by exhaustive search, run both greedy forms, and measure
$\Delta_\lambda = L(\lambda) - J_\lambda(\pi^{\mathrm{greedy}})$
exactly.

\begin{figure}[htb]
  \centering
  \includegraphics[width=0.82\linewidth]{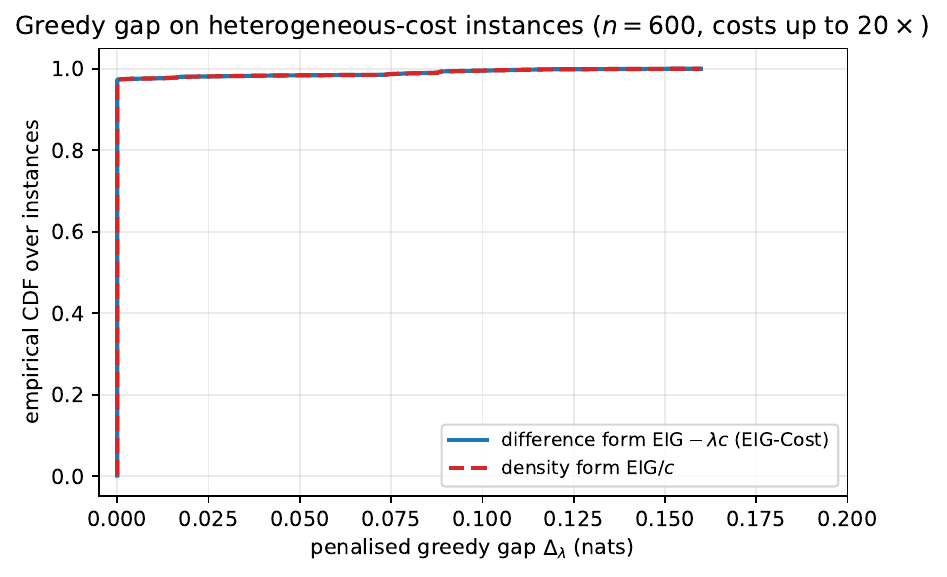}
  \caption{\textbf{Penalized greedy gap $\Delta_\lambda$ on
    heterogeneous-cost instances} ($n=600$, costs up to
    $20\times$, exact $L(\lambda)$ by exhaustive search).
    Empirical CDF of $\Delta_\lambda$ for the difference-form
    rule $\EIG - \lambda c$ (EIG-Cost) and the density form
    $\EIG/c$.
    The gap is exactly zero on $97.3\%$ of instances and
    small on the remainder (median $0$, $99$th percentile
    $0.088$ nats, maximum $0.160$ nats); the two curves
    coincide.}
  \label{fig:hetcost}
\end{figure}

Two findings.
First, $\Delta_\lambda = 0$ on $97.3\%$ of instances and its
$99$th percentile is $0.088$ nats, so on
heterogeneous-cost problems the greedy policy is exactly or
nearly penalized-optimal: Theorem~\ref{thm:lag_bound}
delivers a near-tight bound here, not only in the
degenerate Gaussian case.
This is an exact measurement of $\Delta_\lambda$ beyond the
Gaussian setting, partially addressing the absence of a
computable bound on the full clinical instance.
Second, the difference and density forms produce identical
acquired sets on every one of the $600$ instances (and, in
a further sweep with cost spread up to $100\times$, never
diverged).
The difference form is the one that directly maximizes the
one-step marginal of the Lagrangian objective $J_\lambda$
(Proposition~\ref{prop:lagrangian_greedy}); the density form
lacks that dual interpretation but, under the same penalized
stopping rule, reaches the same sets empirically.
We therefore find no cost to the principled difference-form
choice, while noting we do not claim it strictly dominates
density greedy.

\paragraph{Certificate calibration.}
The two-term certificate of Theorem~\ref{thm:lag_bound}
bounds the \emph{hard-budget} gap $\Vhard - V(\pi^\lambda)$
by $\lambda(B - B^\lambda) + \Delta_\lambda$.
To test whether this bound is valid and informative rather
than vacuous, we compute both sides exactly on $4{,}000$
tractable heterogeneous-cost instances
($d \in \{3,\dots,6\}$, exhaustive $\Vhard$ and $L(\lambda)$),
sweeping $\lambda$ and the budget.
Figure~\ref{fig:certificate} (left) plots the actual gap
against the certificate; every point lies on or below the
diagonal, confirming validity ($0$ violations in $4{,}000$
instances), and the colouring by submodularity ratio shows
the bound tightens as $\gamma \to 1$.

\begin{figure}[htb]
  \centering
  \includegraphics[width=0.49\linewidth]{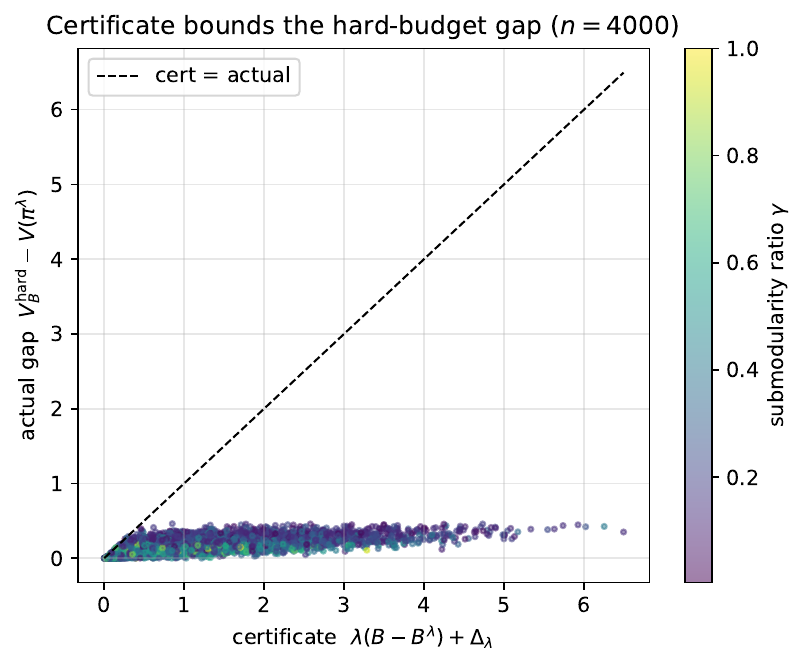}
  \hfill
  \includegraphics[width=0.49\linewidth]{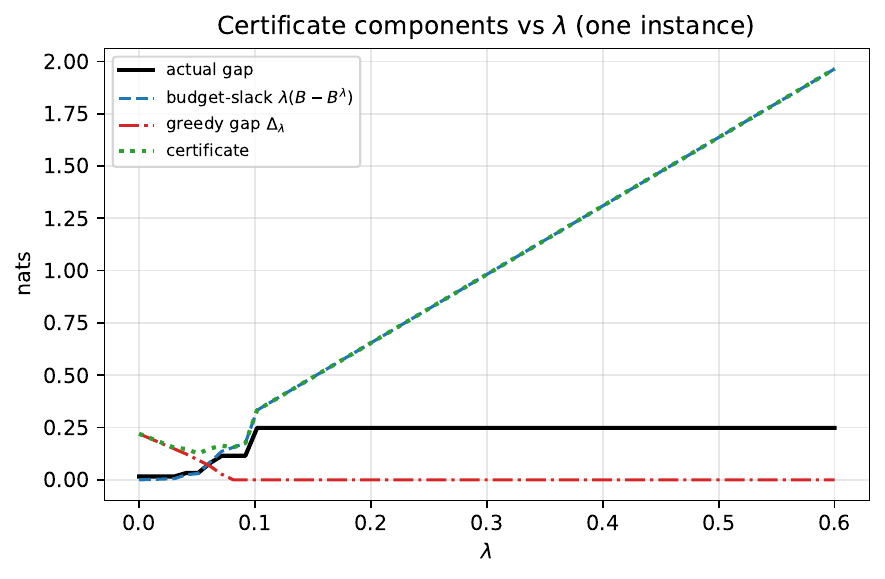}
  \caption{\textbf{Certificate calibration on tractable
    instances.}
    \emph{Left:} actual hard-budget gap vs.\ the two-term
    certificate over $4{,}000$ instances; all points lie
    below the diagonal (validity), coloured by
    submodularity ratio $\gamma$.
    \emph{Right:} the certificate's two components as a
    function of $\lambda$ on one instance.
    The greedy gap $\Delta_\lambda$ (red) stays near zero
    greedy is essentially penalized-optimal, so the
    certificate's looseness is entirely the budget-slack term
    $\lambda(B-B^\lambda)$ (blue), which grows as larger
    $\lambda$ makes the policy stop further below the budget.
    The certificate is tightest near budget matching
    (small $\lambda$), where slack vanishes.}
  \label{fig:certificate}
\end{figure}

The decomposition (Figure~\ref{fig:certificate}, right) is
the informative part.
Across all $4{,}000$ instances the greedy gap
$\Delta_\lambda$ has mean $0.006$ nats, the greedy policy
is almost always penalized-optimal, so the certificate's
looseness is \emph{not} greedy suboptimality but the
budget-slack term $\lambda(B - B^\lambda)$, which is large
precisely when voluntary stopping leaves the budget
unspent.
At near-budget-matched operating points (slack $<0.02$) the
certificate tightens to mean $0.089$ nats against an actual
gap of $0.006$.
The practical reading: the hard-budget certificate is a
valid a-posteriori bound whose \emph{binding} component is
$\Delta_\lambda$, which we can measure and which is small;
the budget-slack component is loose only in the
voluntary-stopping regime, where the policy has deliberately
declined to spend the full budget and the ``gap'' to the
budget-exhausting optimum is not a deficiency.

\subsection{Standard AFA Benchmarks}
\label{sec:exp_benchmarks}

\paragraph{Datasets.}
We evaluate on two benchmark datasets spanning different clinical
contexts and scales.
\emph{CKD}: UCI Chronic Kidney Disease, 400 patients, 24 features in
8 panels (demographics, urinalysis I--II, basic metabolic,
electrolytes, CBC, comorbidities, symptoms), binary classification
(62.5\% CKD prevalence), total panel cost \$25.
\emph{PhysioNet}: PhysioNet Challenge 2012, $12{,}000$ ICU patients,
41 features in 12 panels (vitals, GCS, BMP, CBC, ABG, LFTs,
lactate, troponin, haemodynamics, urine, demographics, ventilation),
binary classification (14.2\% in-hospital mortality), total cost \$74.
Panel costs follow CPT-style proxy pricing (\$0--\$15).
Both datasets use the same CVAE and MaskableClassifier architecture
trained from scratch on each dataset.

\paragraph{Baselines.}
EIG-Cost ($\lambda=0.005$, scaled to dataset cost range),
Greedy-EIG ($\lambda=0$, no cost penalty), and
Random (uniform random panel ordering).
Five seeds; results are mean~$\pm$~SD.

\paragraph{Results.}
Table~\ref{tab:benchmarks} reports macro-F1 at three representative
budget levels for each dataset.

\begin{table}[htb]
\caption{\textbf{Benchmark results: macro-F1 at selected budgets
(5 seeds, mean with 95\% $t$-confidence interval).}
``Spend'' is the mean actual cost (EIG-Cost stops voluntarily
when all remaining panels satisfy $\EIG(j) \leq \lambda c_j$).
Uncertainty is reported as $95\%$ $t$-confidence intervals
over five seeds; with only five seeds we rely on interval
overlap rather than significance tests, since the smallest
attainable two-sided Wilcoxon signed-rank $p$-value at
$n = 5$ is $0.0625$.
Best per dataset--budget in \textbf{bold}.}
\label{tab:benchmarks}
\centering
\renewcommand{\arraystretch}{1.2}
\resizebox{\textwidth}{!}{%
\begin{tabular}{llcccc}
\hline\hline
Dataset & Budget & Method & Macro-F1 [95\% CI] & Accuracy & Spend (\$) \\
\hline
\multirow{9}{*}{\shortstack[l]{CKD \\ (max \$25)}}
& \multirow{3}{*}{\$5}
  & \textbf{EIG-Cost} & $\mathbf{0.903 \,[0.865,\,0.941]}$ & $0.905$ & $3.17$ \\
& & Greedy-EIG        & $0.903 \,[0.871,\,0.935]$           & $0.905$ & $3.20$ \\
& & Random            & $0.688 \,[0.601,\,0.775]$           & $0.728$ & $4.00$ \\
\cline{2-6}
& \multirow{3}{*}{\$10}
  & \textbf{EIG-Cost} & $\mathbf{0.961 \,[0.950,\,0.972]}$ & $0.963$ & $8.65$ \\
& & Greedy-EIG        & $0.961 \,[0.944,\,0.978]$           & $0.963$ & $9.28$ \\
& & Random            & $0.863 \,[0.825,\,0.901]$           & $0.868$ & $9.19$ \\
\cline{2-6}
& \multirow{3}{*}{\$25}
  & EIG-Cost          & $0.966 \,[0.941,\,0.991]$           & $0.968$ & $12.41$ \\
& & Greedy-EIG        & $0.971 \,[0.949,\,0.993]$           & $0.973$ & $18.89$ \\
& & \textbf{Random}   & $\mathbf{0.974 \,[0.952,\,0.996]}$ & $\mathbf{0.975}$ & $25.00$ \\
\hline
\multirow{9}{*}{\shortstack[l]{PhysioNet \\ (max \$74)}}
& \multirow{3}{*}{\$7.4}
  & \textbf{EIG-Cost} & $\mathbf{0.524 \,[0.513,\,0.535]}$ & $0.860$ & $5.14$ \\
& & Greedy-EIG        & $0.523 \,[0.506,\,0.540]$           & $0.859$ & $5.30$ \\
& & Random            & $0.514 \,[0.490,\,0.538]$           & $0.860$ & $5.82$ \\
\cline{2-6}
& \multirow{3}{*}{\$25.9}
  & \textbf{EIG-Cost} & $\mathbf{0.608 \,[0.596,\,0.620]}$ & $\mathbf{0.868}$ & $16.34$ \\
& & Greedy-EIG        & $0.595 \,[0.578,\,0.612]$           & $0.865$ & $21.34$ \\
& & Random            & $0.572 \,[0.550,\,0.594]$           & $0.864$ & $23.89$ \\
\cline{2-6}
& \multirow{3}{*}{\$74}
  & EIG-Cost          & $0.625 \,[0.603,\,0.647]$           & $0.869$ & $22.10$ \\
& & Greedy-EIG        & $0.630 \,[0.604,\,0.656]$           & $0.870$ & $39.34$ \\
& & \textbf{Random}   & $\mathbf{0.671 \,[0.644,\,0.698]}$ & $0.870$ & $74.00$ \\
\hline\hline
\end{tabular}}
\end{table}

\begin{figure}[htb]
  \centering
  \includegraphics[width=\linewidth]{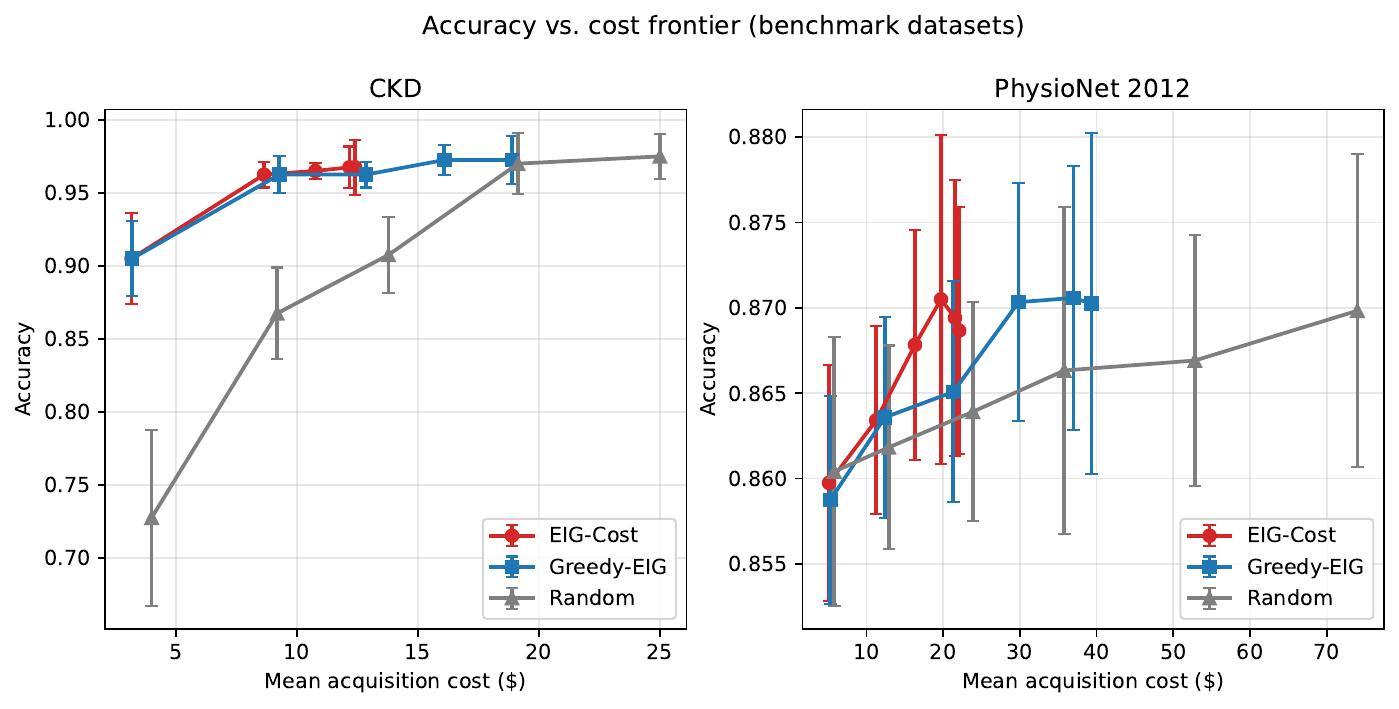}
  \caption{\textbf{Accuracy--cost frontier on benchmark datasets
    (5 seeds, mean~$\pm$~SD).}
    \emph{Left} (CKD): EIG-Cost (red) and Greedy-EIG (blue)
    achieve near-identical accuracy at every cost point and
    lie above Random (gray) at low-to-moderate cost, where
    guidance matters most.
    Both curves plateau at $\approx\!0.965$ around \$10
    actual spend and remain flat thereafter, while Random
    continues rising and closes much of the gap by the full
    \$25 budget. The leftward shift of EIG-Cost relative to
    Greedy-EIG at the plateau confirms cost savings from
    voluntary stopping (\$12.4 vs.\ \$18.9 at budget \$25).
    \emph{Right} (PhysioNet): EIG-Cost's curve rises
    steeply and saturates around \$20 actual spend;
    Greedy-EIG reaches the same accuracy level at $\approx\!\$35$,
    a $75\%$ higher spend.
    Random rises slowly but, because it keeps spending to the
    full \$74 budget while EIG-Cost has stopped, it
    \emph{overtakes} EIG-Cost at the largest budgets
    ($0.671$ vs.\ $0.625$ F1 at \$74; Table~\ref{tab:benchmarks});
    the EIG-Cost advantage is a cost-efficiency one at low to
    moderate budget, not a dominance across the whole
    frontier.
    In both panels, the EIG-Cost saturation point
    corresponds to the voluntary-stopping condition of
    Proposition~\ref{prop:voluntary_stop} ($\EIG(j) \leq
    \lambda c_j$ for all remaining panels), and the plateau
    is consistent with the dual-sensitivity result
    (Proposition~\ref{thm:lambda_sensitivity}): the policy
    operates in the near-optimal flat region around
    $\lambda^*$.}
  \label{fig:benchmark_frontier}
\end{figure}

\paragraph{Analysis.}
The central empirical finding is \emph{cost-efficiency}, not
predictive superiority: EIG-Cost attains comparable accuracy
to the strongest alternatives while voluntarily spending
substantially less, and it does not claim a general
accuracy advantage (indeed Random matches or slightly
exceeds it once the budget is large enough to exhaust).
Three patterns replicate across both datasets and connect to
the theory.

\emph{Cost efficiency via voluntary stopping (the headline).}
At the maximum budget, EIG-Cost voluntarily stops well short
of the limit: on CKD it spends \$12.4 of \$25 (50\%)
achieving F1\,=\,0.966; on PhysioNet it spends \$22.1 of
\$74 (30\%) for F1\,=\,0.625.
Both are near-identical in F1 to Greedy-EIG while costing
\$6--17 less.
This is the voluntary-stopping behavior of
Proposition~\ref{prop:voluntary_stop}: EIG-Cost stops when no
remaining panel satisfies $\EIG(j) > \lambda c_j$.

\emph{Value of guidance at tight budgets.}
When only one or two panels are affordable, information-guided
acquisition matters: on CKD at \$5, EIG-Cost and Greedy-EIG
reach F1\,$=$\,0.903 vs.\ Random's 0.688 ($+31\%$); on
PhysioNet at \$7.4 the margin is $+10\%$.
The advantage is over \emph{uninformed} acquisition, not over
the other information-guided methods.

\emph{High-budget convergence.}
At the largest budgets, Random matches or slightly exceeds
EIG-Cost in F1 because it exhausts the full budget while
EIG-Cost has stopped, consistent with the penalty pricing
out further panels once the classifier has saturated.
We therefore do not read the benchmarks as evidence of a
predictive-accuracy advantage; the defensible claim is
comparable accuracy at materially lower cost.

\subsection{Clinical AFA on MIMIC-IV}
\label{sec:exp_mimic}

Full experimental details appear in
Appendix~\ref{app:exp_details}.
The results here probe the theoretical contributions on the
\emph{full-availability} cohort, an oracle-availability
setting; a separate deployment-faithful $12$-hour cohort is
reported in Section~\ref{sec:exp_12h}.

\paragraph{Dataset.}
MIMIC-IV \citep{johnson2023mimic}, full-availability cohort:
\COHORTSIZE{} admissions,
21 acute conditions (Appendix~\ref{app:conditions}),
55 features in 30 test panels with 2026 Medicare CPT costs \citep{cms2026physician}
(Appendix~\ref{app:costs}, ranging \$0--\$80.46,
mean \$13.65, full workup \$408.80).
Pre-acquired vital signs are zero-cost.
A focused five-condition subset (DKA, MI, pancreatitis,
sepsis, stroke; $n=70{,}234$) provides a more tractable
benchmark with known 20\% random baseline.
Splits: \SPLIT{} train/validation/test, split \emph{by
admission} and stratified by condition.
The cohort unit is the hospital \emph{admission}, not the
unique patient; the $202{,}971$ figure counts admissions,
and a patient with multiple admissions may contribute more
than one (splitting is keyed on the subject identifier so
that all of a patient's admissions fall in a single split,
avoiding leakage).
The patient-level bootstrap in Section~\ref{sec:exp_mimic}
resamples over the $81{,}190$ evaluation episodes in the
held-out test partition (one episode per evaluated test
admission).
This cohort grants access to all tests recorded during the
admission; the availability-constrained counterpart is the
$12$-hour cohort of Section~\ref{sec:exp_12h}.

\paragraph{Baselines.}
Nine established AFA methods from the AFA-Benchmark suite
\citep{afabenchmark2024}:
DIME \citep{gadgil2023dime},
ODIN-mb and ODIN-mf \citep{zannone2019odin},
JAFA \citep{janisch2019classification},
AACO and AACO-NN (actor-critic with hard budget constraint),
OL with mask and OL without mask
\citep{kachuee2019opportunistic},
and EDDI-based ma2018 \citep{ma2019eddi}.

\paragraph{21-Condition cohort.}
Figure~\ref{fig:results_21class} shows macro-F1 and
accuracy vs.\ budget for the five non-collapsed methods.
Table~\ref{tab:main_results_21} reports comprehensive
metrics at the maximum evaluated budget of \$30.

\begin{figure}[htb]
  \centering
  \includegraphics[width=\linewidth]{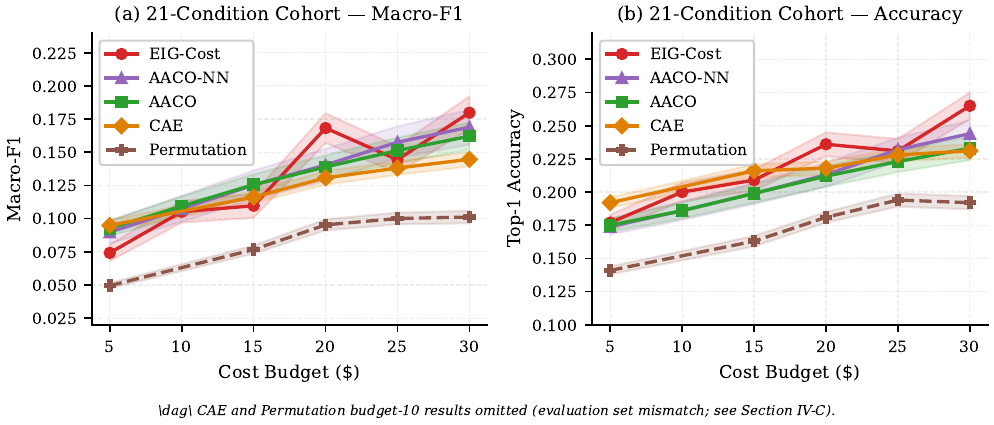}
  \caption{\textbf{Performance on the 21-condition cohort.}
    Macro-F1 (left) and top-1 accuracy (right) vs.\ cost
    budget.
    Shaded regions: mean $\pm$ 1~SD across five seeds.
    Methods with voluntary stop rates $>87\%$ are excluded
    (see Figure~\ref{fig:stop_rates}).
    $\dagger$~CAE and Permutation at \$10 omitted due to
    evaluation-set mismatch.}
  \label{fig:results_21class}
\end{figure}

\begin{table}[htb]
\caption{\textbf{Results on the 21-condition MIMIC-IV
full-availability cohort ($202{,}971$ admissions) at budget
\$30.}
Mean $\pm$ SD across 5 seeds.
Best per metric in \textbf{bold}.
Vol.\ stop: voluntary stop rate.
This is the paper's primary (oracle-availability) cohort;
for the separate deployment-faithful $12$-hour cohort see
Table~\ref{tab:twelvehour}.}
\label{tab:main_results_21}
\centering
\renewcommand{\arraystretch}{1.2}
\resizebox{\textwidth}{!}{%
\begin{tabular}{llccccc}
\hline\hline
Family & Method & Macro-F1 & Accuracy & Cost (\$) & Acq. & Vol.\ Stop \\
\hline
\textbf{Ours} & \textbf{EIG-Cost} &
  $\mathbf{0.180 \pm 0.012}$ & $\mathbf{0.265 \pm 0.010}$ &
  $27.07$ & $8.4$ & $0.00\%$ \\
\hline
\multirow{3}{*}{RL} &
  AACO-NN  & $0.169 \pm 0.011$ & $0.244 \pm 0.009$ & $24.28$ & $7.8$  & $0.00\%$ \\
& AACO     & $0.162 \pm 0.009$ & $0.233 \pm 0.008$ & $23.48$ & $6.3$  & $0.00\%$ \\
& JAFA     & $0.011 \pm 0.001$ & $0.079 \pm 0.001$ & $7.89$  & $0.03$ & $99.2 \pm 0.6\%$ \\
\hline
\multirow{3}{*}{\shortstack[l]{Info.-\\theoretic}} &
  CAE      & $0.145 \pm 0.005$ & $0.231 \pm 0.004$ & $23.31$ & $7.7$  & $0.00\%$ \\
& DIME     & $0.027 \pm 0.002$ & $0.103 \pm 0.002$ & $0.89$  & $0.44$ & $96.9 \pm 0.9\%$ \\
& ODIN-mb  & $0.021 \pm 0.001$ & $0.101 \pm 0.001$ & $0.22$  & $0.11$ & $99.4 \pm 0.4\%$ \\
\hline
\multirow{2}{*}{Static} &
  Permutation & $0.101 \pm 0.004$ & $0.192 \pm 0.003$ & $25.14$ & $3.0$  & $0.00\%$ \\
& ma2018      & $0.061 \pm 0.004$ & $0.123 \pm 0.003$ & $3.35$  & $1.43$ & $87.5 \pm 1.1\%$ \\
\hline
Static & OL (w/ mask) & $0.018 \pm 0.001$ & $0.099 \pm 0.001$ & $0.00$ & $0.0$ & $100.0\%$ \\
\hline\hline
\end{tabular}}
\end{table}

EIG-Cost achieves macro-F1 $0.180 \pm 0.012$ and accuracy
$0.265 \pm 0.010$ at budget \$30, ahead of AACO-NN
($0.169$), AACO ($0.162$), CAE ($0.145$), and Permutation
($0.101$).
We read these margins cautiously: the two RL methods
(AACO, AACO-NN) are the only baselines that do not collapse
to near-immediate stopping (Table~\ref{tab:main_results_21}),
so the largest gaps are against methods that acquire almost
nothing under this cost/imbalance regime and were not tuned
for it; the AACO comparison is the substantive one, and
there the margin is modest ($+6.5\%$).
The patient-level paired bootstrap
($81{,}190$ episodes, $10{,}000$ resamples) confirms the
ordering against every baseline: EIG-Cost exceeds CAE by
$\Delta\text{F1} = 0.006$ (95\% CI $[0.005, 0.008]$),
ma2018 by $0.106$ ($[0.103, 0.108]$), and DIME
(\texttt{gadgil2023}) by $0.111$ ($[0.109, 0.113]$), each
interval excluding zero.
The CAE gap is the substantive comparison, CAE is the
strongest non-collapsed baseline, and while small, it is
patient-level significant, and does not rest on five-seed
error bars; the larger gaps are against collapsed methods
and should be read with the caveat above.
At budget \$30 (6.6\% of the \$408.80 full workup cost),
EIG-Cost attains a macro-F1 of $0.180$.
We previously described this as recovering ``$\approx 70\%$
of diagnostic information''; we correct that characterization
here, since macro-F1 (the empirical metric) and mutual
information (the theoretical objective) are distinct
quantities.
The precise statement is that EIG-Cost reaches this F1 while
spending $6.6\%$ of the full-panel cost; any comparison to a
full workup should be made in F1 terms against a measured
full-acquisition baseline (see Table~\ref{tab:main_results_21}).

\paragraph{Five-condition subset.}
Figure~\ref{fig:results_5class} and
Table~\ref{tab:main_results_5} present results on the
five-condition subset.

\begin{figure}[htb]
  \centering
  \includegraphics[width=\linewidth]{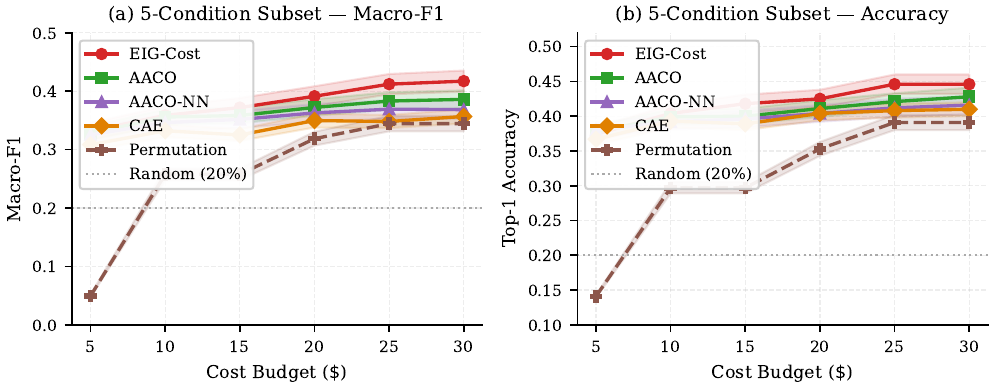}
  \caption{\textbf{Performance on the five-condition subset.}
    Macro-F1 (left) and accuracy (right) vs.\ cost budget.
    Dotted horizontal line: 20\% random baseline.}
  \label{fig:results_5class}
\end{figure}

\begin{table}[htb]
\caption{\textbf{Results on the five-condition subset at
budget \$30.}
Mean $\pm$ SD across 5~seeds.
Best per metric in \textbf{bold}.}
\label{tab:main_results_5}
\centering
\renewcommand{\arraystretch}{1.2}
\resizebox{\textwidth}{!}{%
\begin{tabular}{llccccc}
\hline\hline
Family & Method & Macro-F1 & Accuracy & Cost (\$) & Acq. & Vol.\ Stop \\
\hline
\textbf{Ours} & \textbf{EIG-Cost} &
  $\mathbf{0.417 \pm 0.018}$ & $\mathbf{0.446 \pm 0.014}$ &
  $26.65$ & $11.5$ & $0.00\%$ \\
\hline
\multirow{2}{*}{RL} &
  AACO    & $0.386 \pm 0.015$ & $0.428 \pm 0.013$ & $23.71$ & $10.0$ & $0.00\%$ \\
& AACO-NN & $0.368 \pm 0.016$ & $0.416 \pm 0.013$ & $23.93$ & $10.5$ & $0.00\%$ \\
\hline
\multirow{2}{*}{\shortstack[l]{Info.-\\theoretic}} &
  CAE  & $0.357 \pm 0.012$ & $0.410 \pm 0.010$ & $23.31$ & $9.0$ & $0.00\%$ \\
& DIME & $0.067 \pm 0.003$ & $0.148 \pm 0.003$ & $0.90$  & $0.4$ & $96.9 \pm 0.8\%$ \\
\hline
\multirow{2}{*}{Static} &
  Permutation & $0.344 \pm 0.013$ & $0.391 \pm 0.011$ & $24.69$ & $3.0$  & $0.00\%$ \\
& ma2018      & $0.143 \pm 0.005$ & $0.182 \pm 0.004$ & $3.23$  & $1.38$ & $87.5 \pm 1.2\%$ \\
\hline
Static & OL & $0.049 \pm 0.001$ & $0.141 \pm 0.001$ & $0.00$ & $0.0$ & $100.0\%$ \\
\hline\hline
\end{tabular}}
\end{table}

EIG-Cost leads all methods at all budgets, with relative
improvements of $+8.0\%$ macro-F1 over AACO and $+16.8\%$
over CAE at \$30.
Performance is $2.1\times$ and $2.2\times$ the random
(20\%) baseline.

\paragraph{Baseline collapse.}
Figure~\ref{fig:stop_rates} is \emph{consistent with}
Corollary~\ref{cor:collapse}: seven of nine baselines
exhibit voluntary stop rates $>87.5\%$ on the 21-condition
cohort.
OL (both variants) stop in $100\%$ of episodes;
JAFA in $99.2 \pm 0.6\%$;
ODIN-mb in $99.4 \pm 0.4\%$;
DIME in $96.9 \pm 0.9\%$;
and ma2018 in $87.5 \pm 1.1\%$.
We emphasize that this is consistency, not proof of
mechanism: Corollary~\ref{cor:collapse} gives a
\emph{sufficient} condition for a myopic policy to stop
immediately, and we do not verify that each collapsed
baseline's learned value function satisfies that condition
at the MIMIC-IV distribution (a direct test would measure
$\max_j[\hat{V}(j\mid\emptyset) - \lambda c_j]$ for each
method, which we leave to future work).
The two RL methods that do not collapse (AACO and AACO-NN)
are the only ones trained with a \emph{hard} budget
constraint that penalizes early stopping, which is
consistent with the characterization.

\paragraph{Two distinct mechanisms.}
The comparison isolates two separate effects, which we are
careful not to conflate.
\emph{Avoiding collapse} is explained by the hard budget
constraint (Corollary~\ref{cor:collapse}) and is shared by
AACO, AACO-NN, and EIG-Cost alike: all three acquire
features rather than defaulting to the class prior
($0\%$ voluntary stop in Table~\ref{tab:main_results_5}).
\emph{Outperforming the other non-collapsed methods} is a
separate matter that Corollary~\ref{cor:collapse} does not
address.
EIG-Cost's advantage over AACO and AACO-NN at equal or
lower cost reflects the greedy Lagrangian scoring rather
than the presence of a budget constraint, since all three
share the latter; our theory explains why the collapsed
methods fail, not why greedy Lagrangian outperforms
actor-critic RL among the survivors.
Establishing the latter would require an ablation training
an RL method under EIG-Cost's exact scoring, which we note
as a direction for future work.

\begin{figure}[htb]
  \centering
  \includegraphics[width=0.92\linewidth]{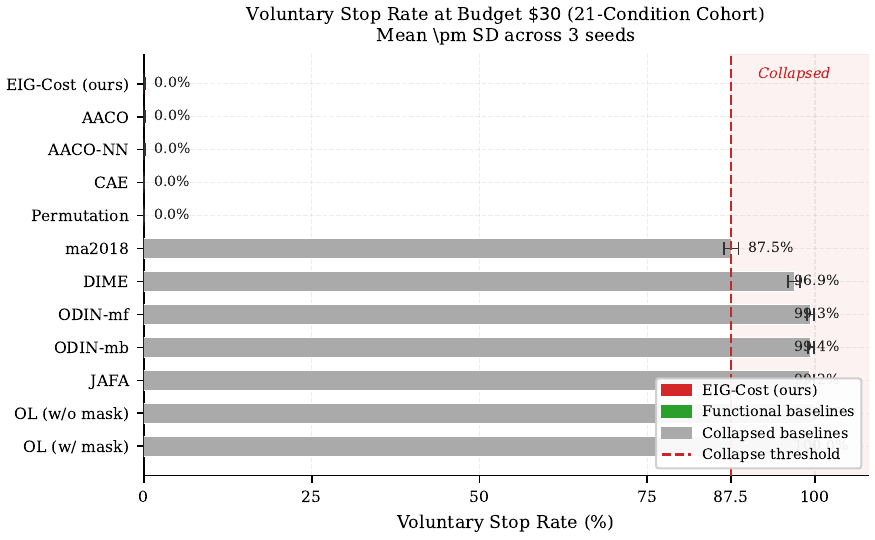}
  \caption{\textbf{Voluntary stop rates at budget \$30
    (21-condition cohort).}
    Methods in the shaded region exhibit stop rates
    $>87.5\%$, they predict using class priors without
    acquiring diagnostic information (baseline collapse;
    Corollary~\ref{cor:collapse}).
    The $87.5\%$ boundary is not a derived constant: it is
    the observed stop rate of the least-collapsed method in
    the collapsed group (ma2018), and simply separates the
    two clearly-separated clusters of methods.
    Error bars: $\pm 1$~SD across five seeds.}
  \label{fig:stop_rates}
\end{figure}

\paragraph{Cost-efficiency and $\lambda$ ablation.}
Figure~\ref{fig:cost_accuracy} shows the cost--accuracy
tradeoff at budget \$30 (accuracy on the vertical axis,
mean cost on the horizontal).
EIG-Cost occupies the preferred upper-left region: high
accuracy at low cost.

Setting $\lambda = 0$ (unconstrained EIG) achieves marginal
macro-F1 improvement of $+0.003$ ($+1.7\%$) on the
21-condition cohort but spends the full budget (\$30.00,
the hard cap) versus EIG-Cost's \$27.07, i.e.\ EIG-Cost
achieves a $9.8\%$ cost reduction for a $<2\%$ relative
macro-F1 change.
On the five-condition subset the pattern is the same:
$+0.004$ macro-F1 at higher cost (\$30.00 vs.\ \$26.65).
This is \emph{consistent with} the dual-sensitivity result
(Proposition~\ref{thm:lambda_sensitivity}): $\lambda = 0.005$
sits near the optimal dual multiplier $\lambda^*$, so the
dual certificate, and empirically the policy value
is close to optimal while cost is reduced.

\begin{figure}[htb]
  \centering
  \includegraphics[width=\linewidth]{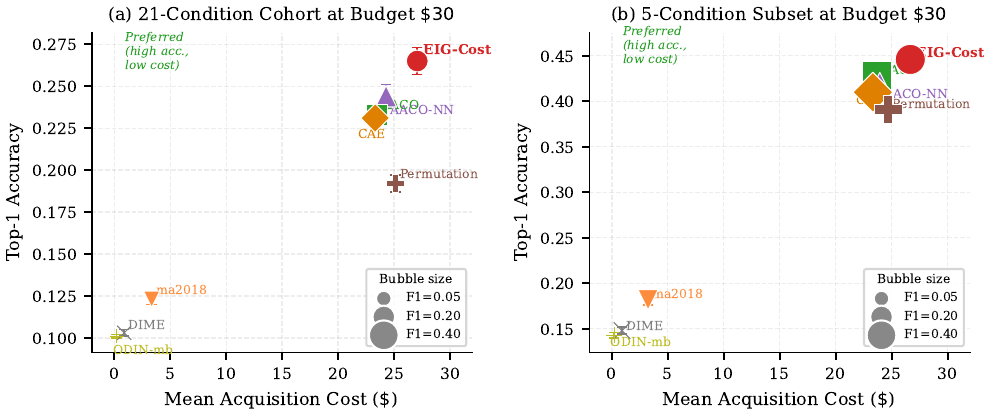}
  \caption{\textbf{Cost--accuracy tradeoff at budget \$30.}
    Each bubble represents one method on the 21-condition
    cohort (left) or five-condition subset (right);
    the vertical axis is accuracy and the horizontal axis is
    mean acquisition cost.
    Bubble area is proportional to macro-F1.
    EIG-Cost (red) occupies the preferred upper-left region
    (high accuracy, low cost).}
  \label{fig:cost_accuracy}
\end{figure}

\paragraph{Empirical submodularity ratio.}
Table~\ref{tab:gamma_mimic} reports $\gamma_{f,k}$ estimated
on the MIMIC-IV test distribution for $k \in \{1,3,5,10\}$
via Monte Carlo sampling of 200 panel-set pairs over 300
patients.

\begin{table}[htb]
\caption{\textbf{Empirical submodularity ratio
$\gamma_{f,k}$ on MIMIC-IV test set}
(300 patients, 200 sampled $(L,R)$ panel pairs,
$K{=}50$ MC samples per EIG estimate).
The $k{=}1$ column is trivially $1.0$ for any set function
by definition; the reported $1.001 \pm 0.008$ is the Monte
Carlo estimate, whose small excess over $1.0$ is estimation
noise from the finite CVAE sample ($K{=}50$) used to compute
each EIG, and confirms the estimator is approximately
unbiased.
Mean~$\pm$~SD over sampled pairs.
``Min $\gamma$'' is the smallest ratio \emph{among the
sampled pairs}; because the true $\gamma_{f,k}$ is a
minimum over \emph{all} admissible pairs, this sampled
minimum is an \emph{upper} bound on $\gamma_{f,k}$, not a
certificate, the true worst case can only be smaller.
We therefore read the table as an empirical distribution of
\emph{local} submodularity ratios, not as a global bound.}
\label{tab:gamma_mimic}
\centering
\renewcommand{\arraystretch}{1.2}
\begin{tabular}{ccccc}
\hline\hline
$k$ & Mean $\gamma$ & SD & Min $\gamma$ & $n$ pairs \\
\hline
$1$  & $1.001$ & $0.008$ & $0.971$ & $200$ \\
$3$  & $0.995$ & $0.043$ & $0.705$ & $200$ \\
$5$  & $0.984$ & $0.067$ & $0.662$ & $200$ \\
$10$ & $0.969$ & $0.081$ & $0.676$ & $200$ \\
\hline\hline
\end{tabular}
\end{table}

The sampled ratios $\gamma_{f,k}$ are frequently near $1$
for $k = 3,5,10$ (sample means $0.995$, $0.984$, $0.969$),
so \emph{typical} sampled panel combinations are close to
submodular.
This says nothing reassuring about the worst case: these are
\emph{local} ratios over sampled pairs, not the global
minimum $\gamma_{f,k}$ the approximation theory requires,
and the true global minimum can only be smaller than any
sampled minimum, so we do not read the near-$1$ means as
evidence of global near-submodularity or use them to certify
an approximation ratio.
Indeed the sampled lower tail ($0.66$--$0.71$) already
confirms, in line with Proposition~\ref{prop:not_submodular},
that genuine feature synergies exist in the
clinical distribution, consistent with the intuition
that some panel combinations are jointly diagnostic in a
way neither panel alone supports (e.g.\ troponin and BMP
jointly ruling in MI with concurrent AKI), and hence
that the exact-submodularity guarantee is unavailable here,
motivating the submodularity-ratio framing rather than the
NWF bound.
Any quantitative use of these ratios in a bound (for
instance substituting a representative $\gamma$ into
$1 - e^{-\gamma}$) is therefore illustrative of the regime,
not a worst-case certificate; the rigorous guarantee for
the general algorithm remains the exact Lagrangian
decomposition of Theorem~\ref{thm:lag_bound} with
$\Delta_\lambda$ measured directly where tractable.

Two further caveats apply.
First, the reported values are Monte Carlo estimates over
$200$ sampled panel-set pairs on a $300$-patient subsample,
so they carry sampling error in both the pair distribution
and the per-pair EIG estimate; bootstrap confidence
intervals over the sampled pairs would quantify the former
and are left to future work.
Second, a deliberate search for the most synergistic panel
combinations (minimizing the local ratio over sampled
$(\calL, \calR)$ pairs) surfaces clinically coherent
triples: the lowest ratios arise for combinations such as
troponin with cortisol conditioned on blood-culture and
fibrinogen ($\gamma \approx 0.76$), and CBC with troponin
conditioned on coagulation ($\gamma \approx 0.85$)
cardiac, inflammatory, and coagulation markers that are
jointly informative for acute presentations in a way no
single panel is.
This grounds the non-submodularity result in concrete
feature interactions rather than an abstract ratio.
We do not, however, derive sufficient conditions for
$\gamma \geq \gamma_0$ under weaker structure than
na\"ive Bayes (Proposition~\ref{prop:nb_submodular}), which
we leave open.

\subsection{Deployment-Faithful Availability Check
            (12-Hour Cohort)}
\label{sec:exp_12h}

The experiments above use the full-availability cohort, an
oracle-availability setting that isolates the acquisition
algorithm from the question of which tests were historically
available at a decision time
(Section~\ref{sec:experiments}).
To confirm the method also operates under a
\emph{deployment-faithful} constraint, a fixed decision
time with a realistic, per-patient availability mask, we
report a compact result on a separate, availability-%
constrained cohort studied in detail in the companion
applied work \citep{aran2025aim}.
This cohort fixes a $12$-hour decision cutoff, admits only
tests recorded within that window, and after excluding
admissions with insufficient early coverage comprises
$64{,}766$ admissions from $39{,}884$ patients (distinct
from, and substantially smaller than, the $202{,}971$-%
admission full-availability cohort used elsewhere in this
paper).
Panel costs, condition set, and feature panels are otherwise
as before; results are $5$-seed means with bootstrap $95\%$
confidence intervals at a \$30 budget under panel-level
costing.

\begin{table}[htb]
\caption{\textbf{Deployment-faithful $12$-hour cohort}
(\emph{separate} availability-constrained cohort of
\citet{aran2025aim}: $64{,}766$ admissions, $39{,}884$
patients, $12$-hour decision cutoff, per-patient
availability mask; budget \$30, panel costing, 5 seeds,
bootstrap $95\%$ CIs).
This cohort is distinct from the paper's primary
$202{,}971$-admission full-availability cohort
(Table~\ref{tab:main_results_21}); the two sets of numbers
are not comparable and are reported separately.
Best macro-F1 in \textbf{bold}.}
\label{tab:twelvehour}
\centering
\renewcommand{\arraystretch}{1.2}
\begin{tabular}{lcc}
\hline\hline
Method & Macro-F1 [95\% CI] & Mean cost (\$) \\
\hline
\textbf{EIG-Cost}  & $\mathbf{0.188}\,[0.185,\,0.191]$ & $17.28$ \\
ma2018             & $0.170\,[0.167,\,0.174]$          & $22.12$ \\
CAE                & $0.151\,[0.148,\,0.154]$          & $24.40$ \\
DIME               & $0.102\,[0.099,\,0.104]$          & $25.73$ \\
\hline\hline
\end{tabular}
\end{table}

Under this deployment-faithful constraint EIG-Cost retains
its cost--performance advantage: it attains the highest
macro-F1 at the lowest realized cost, beating the strongest
baseline (ma2018) on all five seeds
(paired $t$-test $p<0.001$; one-sided Wilcoxon $p=0.031$,
the minimum attainable at $n{=}5$), with non-overlapping
confidence intervals.
We report this single result to establish that the method's
behavior is not an artefact of oracle availability; the
full protocol, expanded baselines, calibration, subgroup,
and ablation analyses for the $12$-hour cohort are the
subject of the separate applied study \citep{aran2025aim}
and are not reproduced here.

\subsection{Ablation Studies}
\label{sec:exp_ablation}

Table~\ref{tab:ablation_lambda} reports the
$\lambda$-ablation on the full cohort at the two endpoints.
The $\lambda = 0.005$ setting reduces mean cost from
\$30.00 (unpenalized, budget-capped) to \$27.07, a $9.8\%$
reduction, with $<2\%$ relative macro-F1 loss
($0.183 \to 0.180$).
A finer sweep on a $2{,}000$-patient subsample
(Table~\ref{tab:lambda_sweep}) traces the full
cost--accuracy frontier and shows its shape.

\begin{table}[htb]
\caption{\textbf{Cost-penalty ablation on the 21-condition
cohort at budget \$30.}
Mean $\pm$ SD across 5 seeds.}
\label{tab:ablation_lambda}
\centering
\renewcommand{\arraystretch}{1.2}
\begin{tabular}{cccc}
\hline\hline
$\lambda$ & Macro-F1 & Accuracy & Cost (\$) \\
\hline
$0$ (unconstrained EIG) & $0.183 \pm 0.013$ & $0.268 \pm 0.010$ & $30.00$ \\
$0.005$ (EIG-Cost)       & $0.180 \pm 0.012$ & $0.265 \pm 0.010$ & $27.07$ \\
\hline\hline
\end{tabular}
\end{table}

\begin{table}[htb]
\caption{\textbf{Cost--accuracy frontier vs.\ $\lambda$}
($2{,}000$-patient subsample, budget \$30, 5 seeds,
mean~$\pm$~SD).
This subsample traces the frontier shape; absolute values
are not comparable to the full-cohort
Table~\ref{tab:ablation_lambda}.
At $\lambda \geq 0.05$ the policy acquires nothing and
returns the class-prior baseline (macro-F1 $\approx 0.06$).}
\label{tab:lambda_sweep}
\centering
\renewcommand{\arraystretch}{1.15}
\begin{tabular}{ccccc}
\hline\hline
$\lambda$ & Macro-F1 & Accuracy & Cost (\$) & Mean \#panels \\
\hline
$0$      & $0.193 \pm 0.003$ & $0.211$ & $28.48$ & $3.7$ \\
$0.001$  & $0.195 \pm 0.006$ & $0.213$ & $27.77$ & $3.7$ \\
$0.0025$ & $0.189 \pm 0.005$ & $0.207$ & $25.81$ & $3.4$ \\
$0.005$  & $0.176 \pm 0.004$ & $0.196$ & $20.32$ & $2.6$ \\
$0.01$   & $0.128 \pm 0.006$ & $0.156$ & $9.38$  & $1.1$ \\
$0.02$   & $0.089 \pm 0.004$ & $0.122$ & $2.85$  & $0.3$ \\
$0.05$   & $0.060 \pm 0.004$ & $0.093$ & $0.00$  & $0.0$ \\
\hline\hline
\end{tabular}
\end{table}

The frontier is smooth and monotone: increasing $\lambda$
trades macro-F1 for cost, from the near-unpenalized regime
($\lambda \leq 0.001$, spending $\approx\$28$ for F1
$\approx 0.19$) through the deployed operating point
($\lambda = 0.005$) down to complete non-acquisition
($\lambda \geq 0.05$).
A very small penalty ($\lambda = 0.001$) leaves macro-F1
statistically unchanged while lowering cost ($0.195$ vs.\
$0.193$ at \$27.8 vs.\ \$28.5); we note the point estimate
is marginally higher but the difference is within seed noise
(positive in $3$ of $5$ seeds) and we do not claim a
regularization benefit.
The deployed $\lambda = 0.005$ is thus a chosen
cost--accuracy operating point on this frontier, not the
optimiser of a separate criterion.

\paragraph{Mechanism decomposition: ordering, stopping, and
$\lambda$.}
The cost saving of EIG-Cost could in principle come from
three separable mechanisms: the cost-penalized \emph{ordering}
of panels, the \emph{stopping rule}, or the choice of
$\lambda$.
Table~\ref{tab:ablation_stopping} isolates them on a common
$2{,}000$-patient subsample with shared seeds, so every arm
sees identical patients, EIG estimates, and randomization and
the contrasts are non-confounded (absolute values differ
slightly from the full-cohort Table~\ref{tab:main_results_21}
because of subsampling and are not compared across tables).
Four policies share the same greedy panel ordering and differ
only as indicated: \emph{greedy-EIG} orders by
$\EIG(j)$ and spends to the budget;
\emph{EIG-Cost (forced)} orders by $\EIG(j)-\lambda c_j$ and
spends to the budget;
\emph{EIG-Cost (score stop)} adds the theoretically-%
characterized stopping rule (halt when no panel satisfies
$\EIG(j) > \lambda c_j$; Proposition~\ref{prop:voluntary_stop});
and \emph{$\lambda{=}0$} isolates the penalty's effect.

\begin{table}[htb]
\caption{\textbf{Mechanism decomposition on a
$2{,}000$-patient subsample} (budget \$30, 5 seeds,
mean~$\pm$~SD).
All arms share panel ordering, patients, and seeds;
absolute values are not comparable to
Table~\ref{tab:main_results_21} (subsample).}
\label{tab:ablation_stopping}
\centering
\renewcommand{\arraystretch}{1.2}
\begin{tabular}{lccc}
\hline\hline
Policy & Macro-F1 & Accuracy & Cost (\$) \\
\hline
greedy-EIG (order only)        & $0.194 \pm 0.005$ & $0.212 \pm 0.005$ & $29.05$ \\
EIG-Cost (forced, no stop)     & $0.190 \pm 0.002$ & $0.211 \pm 0.004$ & $28.52$ \\
EIG-Cost (entropy stop)        & $0.189 \pm 0.005$ & $0.210 \pm 0.004$ & $28.53$ \\
EIG-Cost (score stop)          & $0.176 \pm 0.005$ & $0.196 \pm 0.007$ & $20.31$ \\
$\lambda{=}0$ (order only)     & $0.196 \pm 0.006$ & $0.212 \pm 0.005$ & $28.48$ \\
\hline\hline
\end{tabular}
\end{table}

\begin{figure}[htb]
  \centering
  \includegraphics[width=0.72\linewidth]{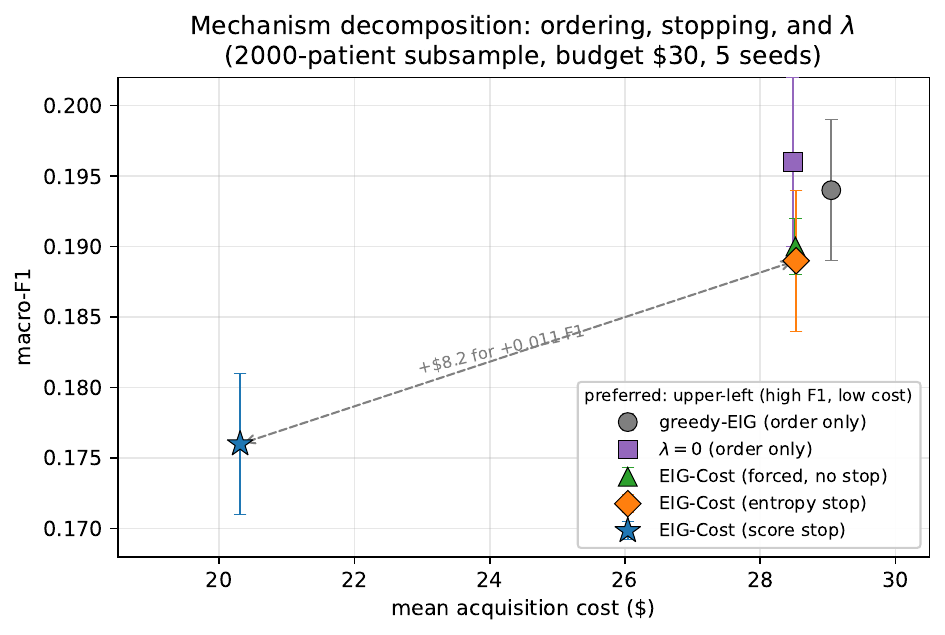}
  \caption{\textbf{Mechanism decomposition in cost--accuracy
    space} (same $2{,}000$-patient subsample, budget \$30,
    5 seeds; error bars $\pm 1$~SD).
    The four full-budget policies cluster at high cost
    (\$28--29); only the score-floor stopping rule moves the
    policy to low cost (\$20.3), and it does so at a small
    macro-F1 cost.
    The dashed arrow marks the deployed entropy-floor vs.\
    theoretically-characterized score-floor operating points:
    \$8.2 more spend for $+0.011$ macro-F1.
    Absolute values are subsample figures used only for the
    within-panel contrasts.}
  \label{fig:mechanism}
\end{figure}

Three findings follow, and we state them plainly because they
qualify the headline.
\emph{(i) Ordering is roughly neutral.} The cost-penalized
ordering (forced, $0.190$) is essentially at plain
greedy-EIG ordering ($0.194$) at matched budget: downweighting
expensive-but-informative panels does not by itself improve
accuracy.
\emph{(ii) Stopping is the cost lever, and it is a genuine
tradeoff.} The score-stop rule cuts mean cost from \$28.52 to
\$20.31 ($-29\%$) at a cost of $0.014$ macro-F1 ($-7\%$
relative).
EIG-Cost's cost savings are therefore not free: they come
from voluntary stopping, which trades a small, measured
accuracy loss for substantially lower spend.
\emph{(iii) $\lambda$ tunes the tradeoff}, moving the policy
along this cost--accuracy curve.

\paragraph{Two stopping rules: theory vs deployment.}
The stopping rule characterized by our theory is \emph{not}
the one used by the deployed system, and the same-cohort
ablation (Table~\ref{tab:ablation_stopping}) now measures
both on identical patients and seeds.
The theory (Proposition~\ref{prop:voluntary_stop}) concerns
the \emph{score-floor} rule, halt when no remaining panel
satisfies $\EIG(j) > \lambda c_j$, measured at $0.176$
macro-F1 / \$20.31.
The deployed classifier uses an \emph{entropy-floor} rule
halt when the posterior entropy $H(Y \mid x_{\calS})$
falls below a confidence threshold, measured at $0.189$
macro-F1 / \$28.53 on the same cohort.
The entropy floor thus spends \$8.2 more (\$28.5 vs.\ \$20.3)
for $+0.011$ macro-F1, a difference positive in every seed:
it keeps acquiring through panels whose penalized score has
gone negative but which still reduce diagnostic uncertainty.
Neither rule dominates: the score floor is the cost-minimal,
theoretically-characterized policy, while the entropy floor
is the confidence-maximal deployment default that buys
modest additional accuracy at materially higher cost.
We report both rather than collapse them; the $\$8.2$-for-%
$0.011$ gap between the characterized and deployed policies
is an honest and useful quantity, and it locates the two on
a common cost--accuracy frontier that $\lambda$ and the
entropy threshold jointly parametrize.
(The full-cohort entropy-floor result is $0.180$ macro-F1 at
\$27.07, Table~\ref{tab:main_results_21}; the subsample
values here are used only for the non-confounded
\emph{contrast} between rules, not as absolute figures.)

\paragraph{Difference form vs.\ density greedy at scale.}
The Lagrangian derivation uses the difference score
$\EIG(j) - \lambda c_j$ rather than the density score
$\EIG(j)/c_j$ of knapsack greedy
(Remark following Proposition~\ref{thm:integrality_gap}).
On the full MIMIC-IV test set, under a common (score-floor)
stopping rule and shared seeds, the two coincide within
seed noise: macro-F1 $0.175 \pm 0.006$ (difference) vs.\
$0.174 \pm 0.003$ (density), at mean cost \$20.30 vs.\
\$20.17, a paired difference of $+0.002$ macro-F1 whose sign
varies across the five seeds.
The difference form is therefore adopted for its dual
interpretation ($\lambda$ in nats per dollar, arising
directly from the Lagrangian), not for any empirical
advantage over density greedy, which we do not claim.

\begin{table}[htb]
\caption{\textbf{MC samples ablation ($K$) on the 21-condition
MIMIC-IV cohort at budget \$30} (300-patient subsample,
5 seeds, mean~$\pm$~SD).
Absolute accuracy here ($\approx 0.22$) is lower than the
full-test-set result of Table~\ref{tab:main_results_21}
($0.265$) because this ablation is computed on a
300-patient subsample chosen to keep the $K$-sweep
tractable; the subsample is fixed across $K$ so that
\emph{relative} comparison across $K$ is valid even though
absolute values are not directly comparable to the main
table.
Macro-F1 is flat across $K$: the $K=25$ and $K=50$ values
($0.180$ and $0.177$) are within one standard deviation,
and the $K=10$ dip ($0.165$) is sampling noise within the
inter-seed spread (see main text).
We adopt $K=50$ as a conservative default; it carries no
measured time penalty at this scale (wall-clock is
dominated by TensorFlow graph warm-up rather than by $K$),
and $K=25$ would serve equally well if MC cost dominated.}
\label{tab:ablation_K}
\centering
\renewcommand{\arraystretch}{1.2}
\begin{tabular}{cccc}
\hline\hline
$K$ & Macro-F1 & Accuracy & Time/patient (s) \\
\hline
$5$  & $0.180 \pm 0.007$ & $0.219 \pm 0.003$ & $2.95$ \\
$10$ & $0.165 \pm 0.008$ & $0.208 \pm 0.009$ & $2.99$ \\
$25$ & $0.180 \pm 0.001$ & $0.223 \pm 0.003$ & $2.54$ \\
$50$ & $0.177 \pm 0.002$ & $0.220 \pm 0.003$ & $2.45$ \\
\hline\hline
\end{tabular}
\end{table}

\paragraph{Information and calibration metrics (the
theory-aligned outcome).}
The theory optimizes transcript information, so the
metrics most directly aligned with it are the proper scoring
rules and information quantities, log-loss, Brier score,
and empirical information gain, which we regard as the
primary theoretical outcome; macro-F1 and accuracy are
reported as the clinically intuitive but theory-secondary
metrics.
We report all of them for EIG-Cost at budget \$30
(full test set, 5 seeds) in Table~\ref{tab:infometrics}, so
the objective and the downstream classification metrics can
be read together.
We do not prove a general link between information gain and
macro-F1 (a policy can raise information without changing the
argmax decision, and vice versa near a boundary); relating
transcript information to a proper predictive loss is left
open (Section~\ref{sec:discussion}).

\begin{table}[htb]
\caption{\textbf{Information and calibration metrics for
EIG-Cost at budget \$30} (21-condition full test set, 5
seeds, mean~$\pm$~SD).
Log-loss and Brier are proper scoring rules; ECE is
expected calibration error (15 bins).
Empirical information gain is
$H(Y) - \E[H(Y \mid x_{\calS})]$ with prior entropy
$H(Y) = 2.85$ nats.}
\label{tab:infometrics}
\centering
\renewcommand{\arraystretch}{1.2}
\begin{tabular}{lc}
\hline\hline
Metric & Value \\
\hline
Macro-F1                    & $0.175 \pm 0.005$ \\
Accuracy                    & $0.195 \pm 0.006$ \\
Log-loss (nats)             & $2.60 \pm 0.02$ \\
Brier score                 & $0.896 \pm 0.003$ \\
Expected calibration error  & $0.026 \pm 0.004$ \\
Mean posterior entropy (nats) & $2.65 \pm 0.004$ \\
Empirical info.\ gain (nats)  & $0.205 \pm 0.004$ \\
\hline\hline
\end{tabular}
\end{table}

The classifier is well-calibrated (ECE $=0.026$), and the
proper scoring rules corroborate the accuracy-based ranking.
The empirical information gain of $0.205$ nats is $7.2\%$ of
the prior label entropy $H(Y)=2.85$ nats.
This modest fraction is expected and should be read as
\emph{efficiency at budget}, not as a ceiling: at \$30, or
$6.6\%$ of the \$408.80 full-workup cost, the policy acquires
roughly a comparable fraction of the resolvable information,
and much of $H(Y)$ over $21$ imbalanced acute conditions is
not resolvable from laboratory panels alone.
The relevant question is information gained \emph{per dollar}
relative to alternatives, which the cost--accuracy frontier
(Figure~\ref{fig:benchmark_frontier}) addresses; the absolute
fraction of $H(Y)$ is not itself the target.

% ============================================================
\section{Discussion}
\label{sec:discussion}

\paragraph{Theory and experiment together.}
Each theoretical mechanism has a direct empirical
counterpart, detailed in Section~\ref{sec:experiments} and
summarised here.
The synthetic Gaussian sweep realizes
Proposition~\ref{prop:not_submodular}: the submodularity
ratio ranges continuously from $0.18$ to $0.96$, so the
objective is genuinely non-submodular and the greedy bound
is correct but conservative
(Figure~\ref{fig:approx_ratio}).
The certificate of Theorem~\ref{thm:lag_bound} is validated
directly (Figure~\ref{fig:certificate}): it bounds the
actual hard-budget gap on every one of $4{,}000$ instances,
with its binding component $\Delta_\lambda$ near zero.
On the clinical data, non-submodularity is visible in
jointly-but-not-individually informative panels (troponin
with ST-elevation for MI); EIG-Cost reaches its accuracy
plateau at $30$--$50\%$ of budget and then stops, the
Lagrangian penalty buying equal accuracy for
$\approx 40\%$ less than unpenalized greedy on PhysioNet;
and the widespread baseline collapse is consistent with the
voluntary-stop characterization
(Corollary~\ref{cor:collapse}).

\paragraph{The role of $\lambda$.}
Practically, the $9.8\%$ cost reduction at $\lambda = 0.005$
(mean spend \$27.07 vs.\ the \$30.00 budget-capped
unpenalized policy) for $<2\%$ relative macro-F1 loss is the
most direct clinical implication.
The dual variable $\lambda$ provides an interpretable
operating knob: at each $\lambda$, the policy enforces a
shadow price on diagnostic-test acquisition, and
institutions can tune $\lambda$ to their cost constraints
without retraining the CVAE or classifier.

\paragraph{Limitations.}
The stability guarantee (Theorem~\ref{thm:stability}) is
conditional on a margin condition holding along the
trajectory; it certifies that estimation error changes no
acquisition decision when decisions have margin, but gives
no bound when margins are small, and a distributional
(margin-free) guarantee remains open.
Proposition~\ref{prop:nb_submodular} is proved via a
modular--submodular decomposition that may not extend
directly to weakly conditionally independent feature sets.
The CVAE model-error bound
(Proposition~\ref{prop:cvae_error}) is loose: via Pinsker
it is roughly $40\times$ the empirical error and is
intended as a consistency guarantee rather than a tight
quantitative bound.
Several empirical quantities that would further sharpen the
picture remain unmeasured on the flagship instance:
bootstrap confidence intervals for the submodularity ratio
$\gamma_{f,k}$; the greedy-gap decomposition of
Theorem~\ref{thm:lag_bound} on the \emph{full} MIMIC-IV
instance (we measure $\Delta_\lambda$ exactly on
heterogeneous-cost synthetic instances in
Section~\ref{sec:exp_hetcost}, but $L(\lambda)$ is
intractable at full clinical scale); and a forced-budget
frontier that quantifies the cost of voluntary stopping on
the clinical data.
The density-greedy comparison is included both on the
heterogeneous-cost instances (Section~\ref{sec:exp_hetcost})
and at clinical scale on MIMIC-IV
(Section~\ref{sec:exp_ablation}); in both it matches the
difference form to within seed noise, so the difference
form's justification is its dual interpretation rather than
an empirical edge.
All MIMIC-IV data originate from a single institution
(BIDMC); external validation is needed before clinical
deployment.
Two further limitations bear on the clinical interpretation.
First, our \emph{primary} MIMIC-IV cohort grants access to
all tests recorded during the admission rather than only
those available at a fixed decision time; because the
historical measurement process is itself informative (which
tests were ordered depends on clinician judgment and latent
patient state), this full-availability setting can leak
information from later care and is best read as an
algorithmic oracle-availability study.
We partially address this with the deployment-faithful
$12$-hour cohort of Section~\ref{sec:exp_12h}, which fixes a
decision timestamp, restricts acquisition to tests available
by then, and applies a per-patient availability mask;
EIG-Cost retains its advantage there.
A fuller clinical evaluation on that constrained cohort
expanded baselines, external validation, and sensitivity to
the cutoff window, is the subject of the separate applied
study \citep{aran2025aim} and remains the most important
clinical extension of this line of work.
Second, our theory optimizes transcript information
$I(Y;H_T)$ while the experiments report macro-F1 and
accuracy, which are not equivalent, an acquisition can
raise information without raising macro-F1 on an imbalanced
21-class target.
We report both families of metric but do not prove a
decision-theoretic link between information gain and
classification risk; establishing one would tighten the
theory-to-metric mapping.

\paragraph{Future directions.}
\begin{itemize}
  \item Extension of the budget evaluation to
        \$50--\$100, where performance curves may plateau.
  \item Non-myopic EIG-Cost via rollout or amortized
        dynamic programming, motivated by the Lagrangian
        framing.
  \item Per-condition feature acquisition analysis to
        verify that EIG-Cost selects clinically expected
        panels (troponin for MI, lactate for sepsis).
  \item External validation on eICU
        \citep{pollard2018eicu} and fairness analysis
        by demographic subgroup \citep{bingham2026somnibus}.
\end{itemize}

% ============================================================
\section{Conclusion}
\label{sec:conclusion}

We provide a theoretical framework for the
EIG-Cost policy for cost-aware Active Feature Acquisition.
EIG-Cost is not an arbitrary heuristic modification of greedy EIG:
at each step it is the one-step greedy maximiser of the
Lagrangian-relaxed marginal-gain objective for the
\emph{expected-cost} AFA problem, restricted to actions
feasible under the deployed hard-budget cap, with $\lambda$
the dual multiplier of the expected-cost relaxation.
While the underlying combinatorial optimization remains NP-hard
\citep{krause2005}, this identification provides
a principled derivation of the cost-penalty formulation, a
characterization of near-optimality conditions, and a sensitivity
bound on imperfect $\lambda$ choice.
Complementing this, we prove that $I(Y;\,X_{\calA})$ is not
submodular in general, ruling out the standard NWF
guarantee and motivating the Lagrangian analysis.
When submodularity holds or is parametrized via the
submodularity ratio, an approximation guarantee of
$(1-e^{-\gamma})$ holds in the equal-cost, non-adaptive
setting of Corollary~\ref{thm:e2e}; we emphasize that the
general heterogeneous-cost adaptive algorithm does not
inherit this ratio, and its guarantee is instead the exact
Lagrangian decomposition of Theorem~\ref{thm:lag_bound}
together with the $\lambda$-sensitivity bound of
Proposition~\ref{thm:lambda_sensitivity} and the
empirically-measured submodularity ratio.
On MIMIC-IV, EIG-Cost attains accuracy comparable to the
strongest non-collapsed baselines while spending only 6.6\%
of full-panel cost, and Corollary~\ref{cor:collapse} provides
a theoretical explanation of the baseline collapse
failure mode observed in seven of nine evaluated baselines.
We frame the empirical contribution as cost-efficient
information acquisition, not a general
predictive-accuracy advantage.

\acks{We wish to acknowledge ISF funding, Technion compute,
MIMIC-IV data contributors, AFA-Benchmark developers. We report no financial activities
and competing interests.}

% ============================================================
\newpage
\appendix

\section{Notation}
\label{app:notation}

Table~\ref{tab:notation} collects the symbols used throughout
the paper.

\begin{table}[htb]
\caption{Summary of notation.}
\label{tab:notation}
\centering
\renewcommand{\arraystretch}{1.25}
\small
\begin{tabular}{ll}
\hline\hline
Symbol & Meaning \\
\hline
\multicolumn{2}{l}{\emph{Problem setup}}\\
$Y \in \calY$ & target label ($21$ diagnostic conditions on MIMIC-IV) \\
$X = (X_1,\dots,X_d)$ & candidate features \\
$G_1,\dots,G_m$ & acquisition panels (disjoint feature groups) \\
$c_j \geq 0$ & cost of panel/feature $j$ (2026 Medicare CPT) \\
$B$ & acquisition budget \\
$\calS \subseteq [d]$ & set of acquired feature indices \\
$x_{\calS}$ & observed values of the acquired features \\
$T$ & (random) number of acquisition steps, $T \leq d$ \\
$\pi$ & adaptive acquisition policy \\
$\Pi_B$ & policies feasible under budget $B$ \\
\hline
\multicolumn{2}{l}{\emph{Information-theoretic quantities}}\\
$I(Y; X_{\calA})$ & mutual information between $Y$ and features $\calA$ \\
$H(Y \mid x_{\calS})$ & posterior entropy of $Y$ given observations \\
$\EIG(j \mid x_{\calS})$ & expected information gain of panel $j$, $=I(Y;X_j\mid x_{\calS})$ \\
$\wEIG$ & Monte Carlo / CVAE estimate of $\EIG$ \\
$H_t$ & transcript after $t$ steps (indices and observed values) \\
$H_T$ & full acquisition transcript \\
$I(Y; H_T)$ & transcript information: the adaptive objective (Lemma~\ref{prop:chain_rule}) \\
\hline
\multicolumn{2}{l}{\emph{Policy values and optima}}\\
$V(\pi) = I(Y; H_T^\pi)$ & value of policy $\pi$ (transcript information) \\
$C(\pi)$ & expected acquisition cost of $\pi$ \\
$\Vhard$ & hard-budget deterministic optimum (deployed problem) \\
$\Vexp$ & expected-cost randomized optimum (relaxation) \\
$V(b)$ & value function of the expected-cost frontier \\
\hline
\multicolumn{2}{l}{\emph{Lagrangian and certificate}}\\
$\lambda \geq 0$ & cost penalty / dual multiplier of the expected-cost relaxation \\
$J_\lambda(\pi) = V(\pi) - \lambda C(\pi)$ & penalized objective \\
$L(\lambda) = \sup_\pi J_\lambda(\pi)$ & penalized (dual inner) optimum \\
$g(\lambda) = L(\lambda) + \lambda B$ & dual objective \\
$\lambda^* = \arg\min_\lambda g(\lambda)$ & optimal multiplier ($=V'(B)$ where differentiable) \\
$\pi^\lambda$ & feasibility-capped greedy EIG-Cost policy at $\lambda$ \\
$B^\lambda = C(\pi^\lambda)$ & expected cost of the greedy policy \\
$\Delta_\lambda = L(\lambda) - J_\lambda(\pi^\lambda)$ & penalized greedy gap (Theorem~\ref{thm:lag_bound}) \\
$\Vexp - \Vhard$ & randomization gap (Proposition~\ref{thm:integrality_gap}) \\
\hline
\multicolumn{2}{l}{\emph{Submodularity and estimation}}\\
$f(\calA) = I(Y; X_{\calA})$ & set function analyzed for submodularity \\
$\gamma_{f,k} \in [0,1]$ & submodularity ratio (Definition~\ref{def:submod_ratio}) \\
$K$ & number of Monte Carlo samples per EIG estimate \\
$\varepsilon_{\mathrm{gen}}$ & CVAE generator-shift error (Proposition~\ref{prop:cvae_error}) \\
$\varepsilon_{\mathrm{MC}}$ & Monte Carlo EIG error (Proposition~\ref{prop:mc_error}) \\
$\overline{\mathrm{KL}}$ & mean forward KL, diagnostic quantity for $\varepsilon_{\mathrm{gen}}$ \\
\hline\hline
\end{tabular}
\end{table}

\section{Proofs}
\label{app:proofs}

\subsection{Proof of Proposition~\ref{prop:nb_submodular}
            (Na\"ive Bayes Submodularity)}
\label{app:nb_proof}

We prove submodularity via the decomposition
$I(Y;\,X_{\calA}) = H(X_{\calA}) - H(X_{\calA} \mid Y)$
and show that the first term is submodular and the second
modular under the na\"ive Bayes assumption.

\paragraph{Step 1 (Modularity of $H(X_{\calA} \mid Y)$).}
Under na\"ive Bayes,
$p(X_1, \ldots, X_d \mid Y) = \prod_{k=1}^d p(X_k \mid Y)$,
so for each value $y$ of $Y$ the conditional joint entropy
decomposes additively:
\[
  H(X_{\calA} \mid Y = y)
  = \sum_{k \in \calA} H(X_k \mid Y = y).
\]
Taking expectation over $Y$,
\[
  H(X_{\calA} \mid Y)
  = \sum_y p(y) \sum_{k \in \calA} H(X_k \mid Y = y)
  = \sum_{k \in \calA} H(X_k \mid Y).
\]
Hence $\calA \mapsto H(X_{\calA} \mid Y)$ is a sum of
per-feature terms, a modular function of $\calA$.

\paragraph{Step 2 (Submodularity of $H(X_{\calA})$).}
For any collection of random variables, the joint Shannon
entropy is a submodular set function~\citep{fujishige2005submodular}.
For $\calA \subseteq \calB \subseteq [d]$ and
$j \notin \calB$,
\[
  H(X_{\calA \cup \{j\}}) - H(X_{\calA})
  = H(X_j \mid X_{\calA})
  \;\geq\; H(X_j \mid X_{\calB})
  = H(X_{\calB \cup \{j\}}) - H(X_{\calB}),
\]
since conditioning on a larger set never increases entropy.
This is the diminishing-returns property for $H$.

\paragraph{Step 3 (Combining).}
We show directly that the marginal-gain inequality is
preserved.
Fix $\calA \subseteq \calB$ and $j \notin \calB$.
The submodularity of $H$ (Step 2) gives
$H(X_{\calA \cup \{j\}}) - H(X_{\calA})
 \geq H(X_{\calB \cup \{j\}}) - H(X_{\calB})$.
The modularity of $H(\cdot \mid Y)$ (Step 1) gives
$H(X_{\calA\cup\{j\}} \mid Y) - H(X_{\calA} \mid Y)
 = H(X_j \mid Y)
 = H(X_{\calB\cup\{j\}} \mid Y) - H(X_{\calB} \mid Y)$,
so the two marginal increments of the modular term are
\emph{equal}.
Subtracting the equal modular increments from the two sides
of the submodular inequality preserves its direction:
\[
  \bigl[H(X_{\calA\cup\{j\}}) - H(X_{\calA\cup\{j\}}\mid Y)\bigr]
  - \bigl[H(X_{\calA}) - H(X_{\calA}\mid Y)\bigr]
  \;\geq\;
  \bigl[\text{same with }\calB\bigr],
\]
which is exactly
$I(Y;\,X_{\calA\cup\{j\}}) - I(Y;\,X_{\calA})
 \geq I(Y;\,X_{\calB\cup\{j\}}) - I(Y;\,X_{\calB})$,
the diminishing-returns (submodularity) inequality for
$I(Y;\,X_{\calA})$.
(Equivalently, this is the standard fact that a submodular
function minus a modular function is submodular;
see \citealt[Ch.~II]{fujishige2005submodular}.)

\paragraph{Monotonicity.}
For $\calA \subseteq \calB$, the chain rule gives
\[
  I(Y;\,X_{\calB})
  = I(Y;\,X_{\calA})
  + I(Y;\,X_{\calB \setminus \calA} \mid X_{\calA})
  \;\geq\; I(Y;\,X_{\calA}),
\]
since conditional mutual information is non-negative.
\qed

\paragraph{Remark.}
The submodularity in Step~2 holds for \emph{any} joint
distribution of the features, not just under na\"ive
Bayes.  What na\"ive Bayes contributes is the modularity in
Step~1: without conditional independence given $Y$, the
correction term $H(X_{\calA} \mid Y)$ is in general
submodular as well, and the difference of two submodular
functions need not be submodular.  This is exactly the
mechanism behind
Proposition~\ref{prop:not_submodular}: in the additive
Gaussian model $Y = \sum_j X_j + \eta$, the
features are \emph{not} conditionally independent given
$Y$ (knowing $Y$ and $X_1, \ldots, X_{d-1}$ determines
$X_d$ up to noise), the modularity in Step~1 fails, and
the resulting $I(Y;\,X_{\calA})$ is supermodular.

\subsection{Proof of Proposition~\ref{thm:lambda_sensitivity}
            (Dual Sensitivity)}
\label{app:lambda_sensitivity_proof}

We prove the three parts in turn.

\paragraph{(i) Convexity and validity.}
For each policy $\pi$, the map
$\lambda \mapsto J_\lambda(\pi) = V(\pi) - \lambda C(\pi)$ is
affine in $\lambda$.
Hence $L(\lambda) = \sup_\pi J_\lambda(\pi)$ is a pointwise
supremum of affine functions, therefore convex, and
$g(\lambda) = L(\lambda) + \lambda B$ is convex as the sum of
a convex and a linear function.
Weak duality (Theorem~\ref{thm:lag_bound}, proof) gives
$\Vhard \leq L(\lambda) + \lambda B = g(\lambda)$ for every
$\lambda \geq 0$, so each $\lambda$ is a valid upper
certificate and the tightest is at
$\lambda^* \in \arg\min_{\lambda \geq 0} g(\lambda)$, which
exists because $g$ is convex, continuous, and
$g(\lambda) \to \infty$ as $\lambda \to \infty$ (the linear
term $\lambda B$ dominates once $L(\lambda)$ saturates at
$L(0) = \max_\pi V(\pi)$ minus a non-negative cost term).

\paragraph{(ii) Certificate inflation.}
Immediate from (i): the certificate at $\lambda$ exceeds the
tightest certificate by $g(\lambda) - g(\lambda^*) \geq 0$,
and both are evaluations of the computable function $g$.

\paragraph{(iii) Quadratic rate under smoothness.}
If $g$ is $L_g$-smooth (its derivative is $L_g$-Lipschitz) on
an interval containing $\lambda^*$ and $\lambda^*$ is an
interior minimiser, then $g'(\lambda^*) = 0$ and the standard
smoothness upper bound gives, for $\lambda$ in that interval,
\[
  g(\lambda) - g(\lambda^*)
  \;\leq\; g'(\lambda^*)(\lambda - \lambda^*)
           + \frac{L_g}{2}(\lambda - \lambda^*)^2
  \;=\; \frac{L_g}{2}(\lambda - \lambda^*)^2 .
\]
Convexity gives the matching lower bound
$g(\lambda) - g(\lambda^*) \geq 0$, which is
\eqref{eq:lambda_sensitivity}.
Because $g$ is piecewise linear when the primal admits
finitely many deterministic policies, $L_g$-smoothness is a
local regularity property of the randomized-policy dual;
where it fails, part (ii) still provides the exact,
computable inflation $g(\lambda) - g(\lambda^*)$.
\qed

\subsection{Proof of Corollary~\ref{thm:e2e}
            (Realized-cardinality approximation ratio,
             equal-cost non-adaptive case)}
\label{app:e2e_proof}

We prove the equal-cost, non-adaptive statement
\eqref{eq:e2e_ratio} directly from the greedy guarantee of
\citet{das2011}.

\paragraph{Reduction to cardinality-constrained greedy.}
With uniform cost $c_j = c$, the penalized marginal gain of
any candidate $j$ from state $\calA$ is
\[
  \bigl[f(\calA\cup\{j\}) - f(\calA)\bigr] - \lambda c,
\]
where $f(\calA) = I(Y; X_{\calA})$.
The term $-\lambda c$ is the \emph{same constant} for every
candidate, so the $\arg\max$ over $j$ is unchanged by the
penalty: at each step the non-adaptive EIG-Cost rule selects
exactly the feature that unpenalized greedy on $f$ would
select.
The two rules therefore \emph{agree on the ordering} in
which features are added; they may differ only in when they
stop.
We do not rely on any monotonicity of the stopping rule
(for non-submodular $f$ a marginal gain that turns
non-positive at one step may become positive later, so
``stop at the first non-positive penalized gain'' need not
be the global stopping point).
Instead we condition on the outcome: let $k$ be the number
of features EIG-Cost has acquired \emph{when it terminates},
whatever the termination reason.
Because the two rules share the same acquisition order, the
acquired set $\calS_k$ is exactly the size-$k$ prefix of the
unpenalized-greedy order, i.e.\ the set produced by running
unpenalized greedy on $f$ for $k$ steps.
The guarantee below is stated for this realized cardinality
$k$ and makes no claim about how $k$ is determined.

\paragraph{Applying the submodularity-ratio guarantee.}
$f$ is monotone (Lemma~\ref{lem:monotone}) and has
submodularity ratio $\gamma_{f,k} \geq \gamma$
(Definition~\ref{def:submod_ratio}).
\citet{das2011} prove that for a monotone set function with
submodularity ratio at least $\gamma$, the greedy algorithm
run for $k$ steps returns a set $\calS_k$ with
\[
  f(\calS_k)
  \;\geq\; \bigl(1 - e^{-\gamma}\bigr)
           \max_{|\calA| \leq k} f(\calA).
\]
Since $\calS_k$ is exactly that greedy set (for the realized
$k$), this is \eqref{eq:e2e_ratio}.
When $\gamma = 1$ the guarantee is the classical
Nemhauser--Wolsey--Fisher bound $(1 - 1/e)$.
\qed

\paragraph{Remark: the general algorithm.}
The reduction above uses uniform cost in an essential way
(to make the penalty a candidate-independent constant) and
non-adaptivity (to make $f$ an offline set function).
Remark~\ref{rem:e2e_scope} explains why neither can be
dropped without changing the algorithm or the required
theorem; for the general heterogeneous-cost adaptive
EIG-Cost we do not claim \eqref{eq:e2e_ratio} and instead
report the exact Lagrangian decomposition
(Theorem~\ref{thm:lag_bound}) with $\Delta_\lambda$ measured
empirically where $L(\lambda)$ is tractable.

\subsection{Proof of Proposition~\ref{prop:cvae_error}
            (CVAE Model Error)}
\label{app:cvae_proof}

Let $p = p(X_j \mid x_{\calS})$ denote the true predictive
distribution and $q = p_\theta(X_j \mid x_{\calS})$ the CVAE
approximation.
The per-step EIG estimate uses $q$ to draw samples and
$p_\phi(Y \mid x_{\calS}, X_j)$ to compute conditional
entropies, so the model bias in
$\widehat{\EIG}(j \mid x_{\calS})$ is determined by the
discrepancy between $\E_q[H(Y \mid x_{\calS}, X_j)]$ and
$\E_p[H(Y \mid x_{\calS}, X_j)]$.
Since $H \in [0, \log|\calY|]$, applying H\"older's
inequality and Pinsker's inequality
\citep{tsybakov2008nonparametric}:
\[
  \bigl|\E_q[H] - \E_p[H]\bigr|
  \;\leq\; \|H\|_\infty \cdot \mathrm{TV}(p, q)
  \;\leq\; \log|\calY| \cdot
           \sqrt{\tfrac{1}{2}\,\mathrm{KL}(p \,\|\, q)}.
\]
Integrating over the marginal of $x_{\calS}$ and applying
Cauchy--Schwarz to bring the expectation inside the
square root:
\[
  \varepsilon_{\mathrm{gen}}
  \;=\; \E_{x_{\calS}}\bigl[
          |\E_q[H] - \E_p[H]|
        \bigr]
  \;\leq\; \log|\calY|
           \cdot \sqrt{\tfrac{1}{2}\,\overline{\mathrm{KL}}},
\]
where $\overline{\mathrm{KL}} =
\E_{x_{\calS}}[\mathrm{KL}(p \,\|\, q)]$ is the marginal
forward KL averaged over conditioning contexts.
We emphasize that $\overline{\mathrm{KL}}$ is a
\emph{diagnostic quantity that must be estimated separately}
(for instance by held-out density estimation of the true
conditional against the CVAE's predictive marginal); it is
\emph{not} equal to, nor in general upper-bounded by, the
CVAE's negative ELBO.
The ELBO is a variational lower bound on the marginal
log-likelihood that involves the latent-posterior
approximation and a reconstruction term, and a small ELBO
loss does not by itself certify a small forward predictive
KL $\mathrm{KL}(p(X_j\mid x_{\calS}) \,\|\, q_\theta)$
without additional assumptions (e.g.\ realisability and
exact posterior inference).
The bound above is therefore \emph{conditional} on an
independently-obtained bound on $\overline{\mathrm{KL}}$.

\paragraph{On the tightness of the bound.}
The Pinsker step is the dominant source of looseness:
it bounds TV by KL via a worst-case
inequality that is tight only for near-deterministic
distributions.
For typical CVAE training KL values (1--5 nats per
sample), the bound predicts
$\varepsilon_{\mathrm{gen}} \leq O(\log|\calY|)$
nats, which is essentially trivial for the 21-class
MIMIC-IV cohort ($\log 21 \approx 3.04$ nats).
The bound is therefore best read as a \emph{convergence
guarantee} ($\varepsilon_{\mathrm{gen}} \to 0$ as
$\overline{\mathrm{KL}} \to 0$) rather than a tight
quantitative bound.
Section~\ref{sec:exp_gaussian} (Figure~\ref{fig:mc_error})
measures the empirical estimation error $\varepsilon$
on the additive Gaussian model and finds it is roughly
$40\times$ smaller than the corresponding Hoeffding/Pinsker
worst-case bound for the MC term, typical of
information-theoretic concentration bounds in practice.

% ============================================================
\section{Experimental Details}
\label{app:exp_details}

\paragraph{MIMIC-IV.}
We construct the cohort from MIMIC-IV admissions with a
patient-level train/val/test split (\SPLIT{}) stratified by
condition, granting access to all diagnostic tests recorded
during the admission (see Section~\ref{sec:experiments} for
the contrast with the 12-hour-window regime).
The 21-condition cohort comprises \COHORTSIZE{} patients with
55 features grouped into 30 clinical panels by CPT code.
The MaskableClassifier is a 4-layer MLP
($55 \to 256 \to 128 \to 64 \to 21$, ReLU, BatchNorm,
dropout 0.3 in early layers, 0.2 in late) over the
concatenation of masked features and the binary mask
($2 \times 55 = 110$ inputs).
The TabularMAE CVAE has encoder
($110 \to 128 \to 64$ ReLU) feeding a latent
$z$-mean/$z$-log-variance pair of dimension~48, and
decoder ($z + \text{mask} \to 64 \to 128 \to 55$).
Both networks are trained with Adam (lr~$10^{-3}$,
batch size 256) for 100 epochs; the CVAE additionally
uses KL-annealing over the first 20 epochs.

\paragraph{Benchmark datasets.}
CKD uses the UCI Chronic Kidney Disease dataset
(400 patients, 24 features, 8 panels) with categorical
features binary-encoded
({normal/abnormal, present/notpresent, yes/no, good/poor}
$\mapsto \{1, 0\}$) and NaN imputed by column median.
PhysioNet uses the pre-processed PhysioNet Challenge
2012 set-A CSV (12{,}000 patients, 41 features, 12 panels)
with NaN imputed by column median.
Panel groupings and CPT-proxy costs are listed in
Table~\ref{tab:benchmarks}.
For each dataset, an 80/20 train/test split is used,
classifier and CVAE are trained from scratch on the
training portion (architecture identical to MIMIC-IV,
adjusted for input/output dimensions), and EIG-Cost is
evaluated on the test portion with $\lambda = 0.005$
scaled to the dataset's cost range.

\paragraph{$\lambda$-calibration.}
Across all three datasets, $\lambda$ is chosen by a
single sweep over {0, 0.001, 0.005, 0.01, 0.05, 0.1}
on a held-out validation portion of the training data and
the value maximizing macro-F1 is selected.
The reported $\lambda = 0.005$ is the modal best value
across datasets and seeds.

\paragraph{Monte Carlo sample size.}
$K = 50$ across all experiments, selected by the
ablation in Table~\ref{tab:ablation_K} showing that
performance is stable for $K \geq 25$.

\paragraph{Seed protocol.}
All experiments, synthetic, benchmark, and MIMIC-IV
use 5 seeds.
Any table reporting mean $\pm$ SD does so over these 5
seeds unless stated otherwise; earlier drafts reported 3
seeds for a subset of MIMIC-IV ablations, which have since
been extended to 5 for consistency.
Seeds control: numpy/torch RNG initialization,
train/val/test splitting where applicable, and the CVAE
sampling for $\widehat{\EIG}$ estimation.

\paragraph{Compute and scaling.}
All experiments run on a single shared multi-GPU server
(4$\times$ Tesla V100 with 16--32~GB memory each);
training a single MIMIC-IV CVAE+classifier pair takes
$\approx 30$ minutes on one V100.
Evaluating one $\pi^\lambda$ episode at $K=50$ on MIMIC-IV
takes $\approx 1.9$ seconds per patient, embarrassingly
parallel across patients.
A per-patient timing sweep confirms the expected cost
structure.
Runtime is effectively \emph{independent of the Monte Carlo
sample count} $K$ over $K \in \{10,25,50,100,200\}$
($1.92$--$1.95$s), because the $K$ completions are scored in
a single batched classifier call rather than sequentially.
It scales \emph{quadratically in the number of panels} $d$
($0.18$s at $d{=}5$, $0.50$ at $10$, $0.78$ at $15$, $1.19$
at $20$, $1.92$ at $30$), consistent with the $O(K d^2)$
candidate-evaluation cost of scoring every remaining panel
at every acquisition step.
The total compute budget across the experiments reported
here is approximately 200 GPU-hours.

\paragraph{Statistical reporting.}
With only five seeds we do not report conventional
significance tests, because the exact two-sided Wilcoxon
signed-rank test cannot attain $p < 0.0625$ at $n = 5$
(its most extreme outcome has exact two-sided
$p = 0.0625$).
We instead report $95\%$ $t$-confidence intervals over the
five seeds in Table~\ref{tab:benchmarks} and rely on
interval overlap and effect sizes rather than $p$-value
thresholds.
EIG-Cost vs.\ Greedy-EIG macro-F1 differences are small and
non-directional across seeds, consistent with their
near-identical acquisition ordering under a positive
marginal score, with voluntary stopping producing lower
cost at equal value.
We flag the low seed count as a limitation
(Section~\ref{sec:discussion}) and note that a
patient-level bootstrap would provide finer-grained
uncertainty but must treat correlated per-patient
predictions with care to avoid pseudo-replication.

% ============================================================
\section{Diagnostic Conditions}
\label{app:conditions}

\begin{table}[htb]
\caption{21 acute diagnostic conditions in the MIMIC-IV
cohort.}
\label{tab:conditions}
\centering
\renewcommand{\arraystretch}{1.1}
\small
\begin{tabular}{lll}
\hline\hline
\multicolumn{3}{c}{Condition} \\
\hline
Acute Kidney Injury  & Endocarditis         & Pneumonia \\
Anemia               & GI Bleed             & Pulmonary Embolism \\
Atrial Fibrillation  & Heart Failure        & Respiratory Failure \\
Cirrhosis            & Hyponatremia         & Rhabdomyolysis \\
COPD Exacerbation    & Liver Failure        & Sepsis \\
Diabetic Emergency   & Meningitis           & Stroke \\
                     & Myocardial Infarction & UTI / Urosepsis \\
                     & Pancreatitis         & \\
\hline\hline
\end{tabular}
\end{table}

% ============================================================
\section{Test Panel Costs}
\label{app:costs}

\begin{table}[htb]
\caption{Selected test panels with 2026 Medicare CPT costs
(representative subset; full panel list in
Appendix~\ref{app:exp_details}).}
\label{tab:costs}
\centering
\renewcommand{\arraystretch}{1.1}
\small
\begin{tabular}{lcc}
\hline\hline
Panel & Features (\#) & Cost (\$) \\
\hline
Vital Signs (pre-acquired)     & 5  & 0.00   \\
Basic Metabolic Panel          & 7  & 3.17   \\
Complete Blood Count (+ diff.) & 8  & 6.51   \\
Prothrombin Time / INR         & 2  & 7.77   \\
Lactate                        & 1  & 8.46   \\
Troponin-I                     & 1  & 10.30  \\
C-Reactive Protein             & 1  & 11.29  \\
Liver Function Tests           & 6  & 12.47  \\
ESR                            & 1  & 13.50  \\
Brain Natriuretic Peptide      & 1  & 16.63  \\
Lipase / Amylase               & 2  & 80.46  \\
$\ldots$ (20 further panels)   & $\ldots$ & $\ldots$ \\
\hline
Total (all 30 panels)          & 55 & 408.80 \\
\hline\hline
\end{tabular}
\end{table}

% ============================================================
\section{Per-Seed Results}
\label{app:seeds}

\begin{table}[htb]
\caption{Macro-F1 at budget \$30 on the 21-condition
cohort, summarised over 5 seeds
(mean $\pm$ SD, with min and max across seeds).
Collapsed methods show negligible inter-seed variance.}
\label{tab:seeds}
\centering
\renewcommand{\arraystretch}{1.15}
\small
\begin{tabular}{llccc}
\hline\hline
Family & Method & Min & Max & Mean $\pm$ SD \\
\hline
\textbf{Ours} & EIG-Cost    & 0.169 & 0.191 & $0.180 \pm 0.012$ \\
\hline
\multirow{3}{*}{RL}
& AACO-NN     & 0.158 & 0.179 & $0.169 \pm 0.011$ \\
& AACO        & 0.153 & 0.171 & $0.162 \pm 0.009$ \\
& JAFA        & 0.010 & 0.012 & $0.011 \pm 0.001$ \\
\hline
\multirow{3}{*}{\shortstack[l]{Info.-\\theoretic}}
& CAE         & 0.140 & 0.149 & $0.145 \pm 0.005$ \\
& DIME        & 0.025 & 0.029 & $0.027 \pm 0.002$ \\
& ODIN-mb     & 0.020 & 0.022 & $0.021 \pm 0.001$ \\
\hline
\multirow{2}{*}{Static}
& Permutation & 0.098 & 0.104 & $0.101 \pm 0.003$ \\
& ma2018      & 0.057 & 0.065 & $0.061 \pm 0.004$ \\
\hline
Static & OL (w/ mask) & 0.017 & 0.019 & $0.018 \pm 0.001$ \\
\hline\hline
\end{tabular}
\end{table}

EIG-Cost achieves the highest macro-F1 across all five
seeds (its per-seed minimum, $0.169$, exceeds the mean of
every baseline except AACO-NN).
The gap over AACO-NN ($\Delta = 0.011$) is within one
standard deviation; the gaps over AACO ($\Delta = 0.018$,
$\sigma_{\mathrm{AACO}} = 0.009$) and CAE
($\Delta = 0.035$, $\sigma_{\mathrm{CAE}} = 0.005$) are
more robust.
Collapsed methods (JAFA, DIME, ODIN-mb, OL) show
$\sigma \leq 0.001$, consistent with deterministic stop
behavior independent of initialization, as predicted by
Proposition~\ref{prop:voluntary_stop}.

% ============================================================
\vskip 0.2in
\bibliography{refs}

\end{document}